\documentclass{article}
\PassOptionsToPackage{table}{xcolor}
\usepackage{iclr2027_conference,times}
\iclrfinalcopy

\usepackage[utf8]{inputenc} 
\usepackage[T1]{fontenc}    
\usepackage{hyperref}       
\usepackage{url}            
\usepackage{booktabs}       
\usepackage{amsfonts}       
\usepackage{nicefrac}       
\usepackage{microtype}      
\usepackage[table]{xcolor}         
\usepackage{graphicx}
\definecolor{FitColor}{RGB}{149,67,89} 
\usepackage{wrapfig}
\usepackage{amsmath}
\usepackage{multirow}
\definecolor{baselinecolor}{RGB}{255, 255, 255} 
\definecolor{aurora}{RGB}{255, 255, 255}
\definecolor{ourcolor}{RGB}{221, 235, 247}  
\usepackage{subcaption}
\usepackage{stfloats}
\usepackage{amsthm}

\usepackage{algorithmic}
\usepackage{algorithm}
\theoremstyle{plain}
\newtheorem{theorem}{Theorem}[section]
\newtheorem{proposition}[theorem]{Proposition}
\newtheorem{lemma}[theorem]{Lemma}

\theoremstyle{definition}

\theoremstyle{remark}

\title{Nonlinearity as Rank:
Generative Low-Rank Adapters with Radial Basis Functions}

\author{%
  {\bf Yihao Ouyang}$^{1*}$, 
  {\bf Shiwei Li}$^{1*}$, 
  {\bf Haozhao Wang}$^{1}$, 
  {\bf Xiandi Luo}$^{1}$, 
  {\bf Zhuoqi Hu}$^{1}$, \\ 
  {\bf Jiarui Zhao}$^{1}$, 
  {\bf Yichen Li}$^{1}$, 
  {\bf Ruixuan Li}$^{1}$ \\
  $^1$ Huazhong University of Science and Technology, Wuhan, China \\
  \small $^*$ Equal contribution. 
}

\begin{document}

\maketitle
\lhead{Preprint}

\begin{abstract}
Low-rank adaptation (LoRA) approximates the update of a pretrained weight matrix using the product of two low-rank matrices. However, standard LoRA follows an \textbf{explicit-rank} paradigm, where increasing model capacity requires adding more rows or columns (i.e., \textbf{basis vectors}) to the low-rank matrices, leading to substantial parameter growth. In this paper, we find that these basis vectors exhibit significant parameter redundancy and can be compactly represented by lightweight nonlinear functions. Therefore, we propose \textbf{Generative Low-Rank Adapter (GenLoRA)}, which replaces explicit basis vector storage with nonlinear basis vector generation. Specifically, GenLoRA maintains a latent vector for each low-rank matrix and employs a set of lightweight radial basis functions (RBFs) to synthesize the basis vectors. Each RBF requires far fewer parameters than an explicit basis vector, enabling higher parameter efficiency in GenLoRA. Extensive experiments across multiple datasets and architectures show that GenLoRA attains higher effective LoRA ranks under smaller parameter budgets, resulting in superior fine-tuning performance.  Our code is publicly available at \url{https://anonymous.4open.science/r/GenLoRA}.
\end{abstract}

\section{Introduction}
Large language models (LLMs) have demonstrated unprecedented capabilities across a wide range of tasks \citep{dubey2024LLaMA, deepseek2024deepseek, cai2024internlm2}. To adapt these general-purpose models to specific downstream applications, fine-tuning is essential. However, as model sizes explode, full fine-tuning becomes prohibitive in both computation and memory. 
Parameter-efficient fine-tuning (PEFT) methods \citep{lester2021power,houlsby2019parameter} have thus emerged as a standard paradigm, among which low-rank adaptation (LoRA) \citep{hu2022lora} is widely adopted.
LoRA approximates the update of a pretrained weight $W \in\mathbb{R}^{m \times n}$ as a low-rank decomposition $\Delta W = BA$, where $B \in \mathbb{R}^{m \times r}$ and $A \in \mathbb{R}^{r \times n}$, with $r \ll \min(m, n)$.

The effectiveness of LoRA is closely tied to the rank $r$, as a larger rank generally increases the expressive capacity of the adaptation \citep{expressive-power}.
However, LoRA follows an explicit-rank paradigm: increasing $r$ requires maintaining additional independent basis vectors, corresponding to the columns of $B$ and the rows of $A$.
In modern LLMs, the projection dimensions $m$ and $n$ are typically on the order of thousands, which makes this expansion costly.
Specifically, as the rank increases, the total number of parameters grows linearly with $m+n$, leading to a parameter complexity of $\mathcal{O}(r(m+n))$.
This scaling behavior motivates a natural question: \textbf{\textit{do these high-dimensional basis vectors contain parameter redundancy that can be exploited for more efficient adaptation?}}

To examine the suspected parameter redundancy, we study learnable non-linear functions for basis vector reconstruction.
Specifically, given a pretrained LoRA matrix (i.e., the matrix $B$) containing $r$ basis vectors , we first construct a shared prototype vector by averaging these $r$ basis vectors. We then learn $r$ independent non-linear functions that take this shared prototype as input and reconstruct the original $r$ basis vectors.
In this work, we focus on Radial Basis Functions (RBFs)~\citep{buhmann2000radial},  which are renowned for their powerful approximation capabilities and high parameter efficiency. Additional details of the experimental setup are provided in the Appendix ~\ref{Implementation Details}.
As shown in Figure~\ref{fig:fig1} (\textit{Upper}), the results are compelling: the RBFs successfully recover the global shape and essential patterns of the original basis vectors. Most strikingly, this reconstruction requires very few learnable parameters, amounting to only 6.25\% of the parameters needed to explicitly maintain a single basis vector. This finding affirmatively answers our research question: substantial parameter redundancy indeed exists within the explicit basis vectors, as the information spread across them can be effectively compressed into non-linear transformations. In essence, nonlinearity can serve as a parameter-efficient substitute for rank, offering a direct solution to the identified redundancy, which we term ``Nonlinearity as Rank''.
\begin{wrapfigure}{R}{0.5\textwidth}
\vspace{-1em} 
\centering
\includegraphics[width=\linewidth]{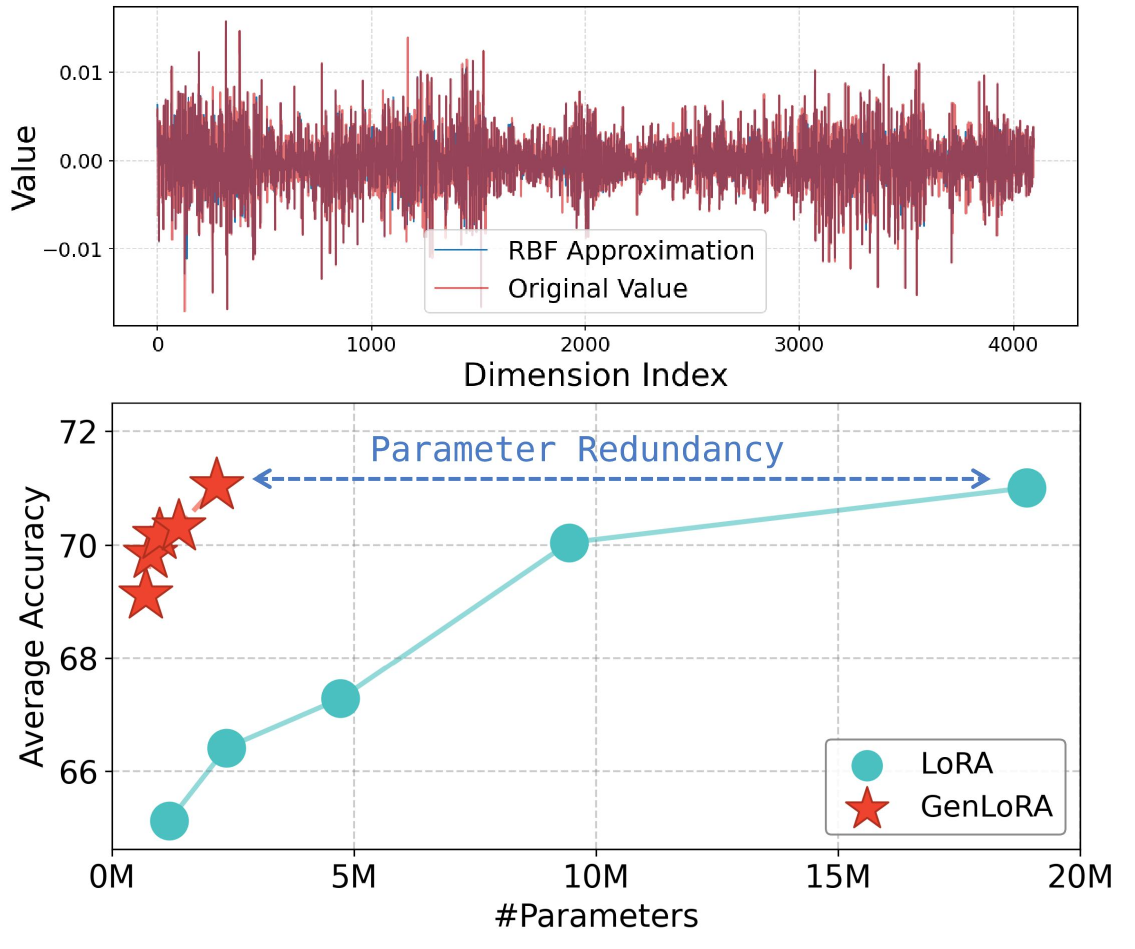}
\caption{\textbf{(\textit{Upper})} The reconstruction result of the first column vector of a pretrained LoRA matrix $B$. The trajectories in \textcolor{FitColor}{\textbf{hue}} illustrate the overlap between the Radial Basis Function (RBF) approximation and the original values.
\textbf{(\textit{Lower})} Accuracy--parameter trade-off on mathematical reasoning tasks with LLaMA3-8B. The five points for each method correspond to ranks $r = \{2, 4, 8, 16, 32\}$.
} 
\label{fig:fig1}
\vspace{-1em} 
\end{wrapfigure}

To address the parameter redundancy described above, we propose \textbf{Generative Low-Rank Adapters (GenLoRA)}, which generate basis vectors through nonlinearity rather than storing them as explicit parameters.
Specifically, GenLoRA maintains a single learnable latent vector for each low-rank matrix ($A$ or $B$) and learns $r$ RBF-based nonlinear generators.
Each generator maps the shared latent vector to one basis vector, and the resulting $r$ generated vectors directly form the matrices $A$ and $B$.
This design allows the effective rank to scale while keeping the parameter complexity at $\mathcal{O}(m + n + r|\theta|)$, where $|\theta|$ denotes the number of parameters in each generator and typically satisfies $|\theta| \ll \min(m, n)$.
As illustrated in Figure~\ref{fig:fig1} (\textit{Lower}), this formulation effectively removes the substantial parameter redundancy in standard LoRA and yields a better accuracy–parameter trade-off.
In addition, theoretical analysis shows that GenLoRA preserves the low-rank structure and has locally bounded gradients under explicit regularity conditions.

The core contributions of this study are summarized as follows:
\begin{itemize}
    \item We reveal significant parameter redundancy in LoRA through the basis vector reconstruction experiment. The results show that basis vectors can be effectively represented by non-linear functions, establishing nonlinearity as a parameter-efficient substitute for rank.
    \item We propose {GenLoRA}, the first generative LoRA framework that synthesizes low-rank matrices from shared latent vectors via RBF generators. This design mitigates the parameter redundancy, allowing effective rank increase with only minimal parameter cost.
    \item Theoretically, GenLoRA preserves the low-rank structure and has locally bounded gradients under explicit regularity conditions. Empirically, extensive experiments demonstrate that GenLoRA attains higher effective LoRA ranks under smaller parameter budgets.
\end{itemize}

\section{Related Work}

\subsection{Low-Rank Adaptation}

LoRA \citep{hu2022lora} reduces fine-tuning overhead by factorizing weight updates as $\Delta W = BA$. Subsequent work primarily explores optimization, linear structural modifications, and nonlinear extensions. For optimization, DoRA \citep{liu2024dora} decouples magnitude and direction updates, LoRA+ \citep{hayou2024lora+} adopts differential learning rates, and PiSSA \citep{meng2024pissa} and LoRA-GA \citep{wang2024lora} improve initialization via SVD of pretrained weights and sampled gradients, respectively. Linear variants address rank constraints through dynamic rank allocation (AdaLoRA \citep{zhang2023adalora}), token-wise projections (TopLoRA \citep{li2025beyond}), or higher-rank constructions using Hadamard products (HiRA \citep{huang2025hira}), Kronecker products (KronA \citep{edalati2025krona}), and block-diagonal stacking (MELoRA \citep{ren2024melora}). NOLA \citep{koohpayegani2023nola} and HyperLoRA \citep{li2025hyperlora} instead reconstruct weights through linear combinations of random or learned bases. Nonlinear extensions include LoRAN \citep{li2024loran}, SineLoRA \citep{ji2024efficient}, and AuroRA \citep{dong2025aurora}. Whereas AuroRA confines nonlinearity to the \textbf{low-dimensional activation space}, GenLoRA introduces it into a \textbf{high-dimensional parameter generation space}, alleviating the information bottleneck with minimal additional parameters. Moreover, rather than generating mixing coefficients over shared bases as in NOLA and HyperLoRA, GenLoRA directly generates basis matrices, enabling more flexible, layer-specific adaptation.

\subsection{Nonlinearity Applications}

Nonlinear approximation has extensively employed rational functions \citep{baker1963theory}, trigonometric functions \citep{edmunds2012properties}, B-splines \citep{de1972calculating}, and radial basis functions (RBFs) \citep{buhmann2000radial}. Recent adaptation methods exploit these forms: NoRA \citep{yin2025don} learns rational activation functions, while SineLoRA \citep{ji2024efficient} and FourierFT \citep{gao2024parameter} employ sinusoidal modulations and Fourier-based parameterizations, respectively. KAN \citep{liu2024kan} further shifts learning from linear weights to parameterizable activation functions, a principle adopted by AuroRA \citep{dong2025aurora} through KAN-based nonlinear MLP layers. In contrast, GenLoRA uses \textbf{learnable nonlinear functions solely for weight generation}, avoiding the overhead of parameter-heavy MLP layers. It further replaces KAN-style B-splines with RBFs to balance approximation capacity and computational efficiency, enabling lightweight high-dimensional parameter generation.
\section{Methodology}
In this section, we introduce \textbf{Generative Low-Rank Adapters (GenLoRA)}, a novel paradigm designed to synthesize high-dimensional basis vectors from a latent vector via lightweight Radial Basis Function (RBF) generators.
\subsection{Nonlinearity as Rank}
We first revisit LoRA from a generative perspective. 
As illustrated in Figure~\ref{fig:fig3}(a), we characterize standard LoRA as adhering to an explicit-rank paradigm, where the low-rank structure is directly constructed from independent and trainable explicit basis vectors. Given a frozen weight matrix $W_0 \in \mathbb{R}^{m \times n}$, standard LoRA learns a low-rank update:
\begin{equation}
\Delta W = BA = \sum_{i=1}^{r} b_i a_i^\top,
\label{eq:lora}
\end{equation}
where the matrices are composed of explicit basis vectors: $B = [b_1, \dots, b_r] \in \mathbb{R}^{m \times r}$ and $A = [a_1, \dots, a_r]^\top \in \mathbb{R}^{r \times n}$. Here, $b_i \in \mathbb{R}^m$ and $a_i \in \mathbb{R}^n$ denote the $i$-th column of $B$ and the $i$-th row of $A$, respectively. This implies that each incremental rank unit incurs a parameter cost of $m + n$, resulting in a total complexity of $\mathcal{O}(r(m + n))$.

\begin{figure*}[t]  
\centering
\includegraphics[width=\textwidth]{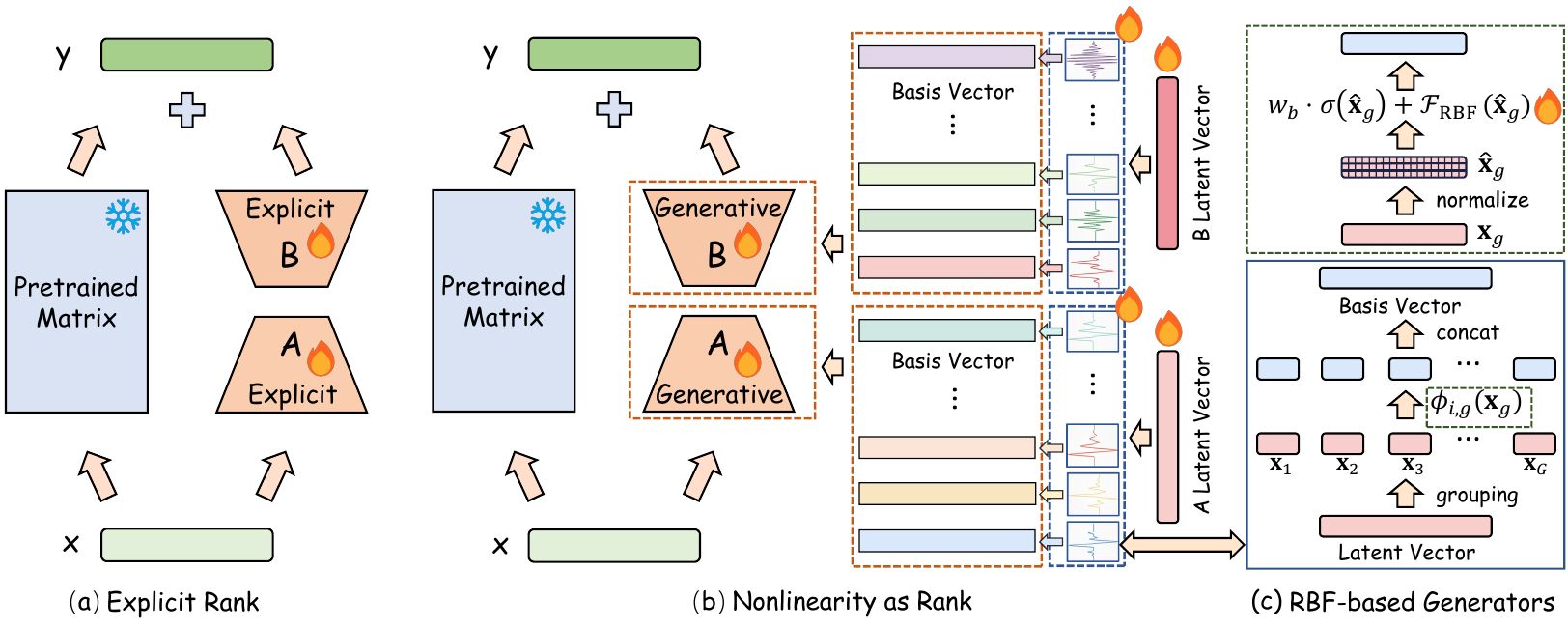} 
\caption{The overall architecture of Generative Low-Rank Adapters (GenLoRA). (a) \textbf{Explicit Rank:} Standard LoRA relies on the explicit-rank paradigm, where model capacity is constrained by the linear dimension of basis vectors. (b) \textbf{Nonlinearity as Rank:} Our proposed paradigm synthesizes basis vectors from latent vectors via generative functions, effectively reducing the parameter cost. (c) \textbf{RBF-based Generators:} The detailed internal workflow of the generator. }
\label{fig:fig3}
\end{figure*}
We propose to treat nonlinearity as a substitute for explicit basis vectors: instead of introducing a massive number of parameters for storage, we synthesize these vectors using a minimal set of parameters as illustrated in Figure~\ref{fig:fig3}(b). Here, ``Nonlinearity as Rank refers to using compact nonlinear generators to reduce the parameter cost of constructing additional rank components. Formally, GenLoRA reparameterizes matrices $B$ and $A$ via latent vectors and nonlinear generators. Let $Z_B \in \mathbb{R}^{m \times 1}$ and $Z_A \in \mathbb{R}^{1 \times n}$ be two learnable latent vectors. Let $F_B = \{f^B_1, \dots, f^B_r\}$ and $F_A = \{f^A_1, \dots, f^A_r\}$ be two sets of lightweight nonlinear generators. Each generator takes the corresponding latent vector $Z_B$ or $Z_A$ as input to synthesize a basis vector:
\begin{equation}
\begin{split}
b_i &= f^{B}_i(Z_B; \theta^{B}_i) \in \mathbb{R}^{m}, \\
a_i &= f^{A}_i(Z_A; \theta^{A}_i) \in \mathbb{R}^{n},
\end{split}
\label{eq:genlora-directions}
\end{equation}
where $\theta_i^B$ and $\theta_i^A$ are the parameters of the nonlinear generators, and we use $\theta$ to denote the generator parameters when no distinction is required.
Stacking these vectors reconstructs the low-rank matrices:
\begin{equation}
\begin{split}
B = [b_1, \dots, b_r], \ A = [a_1, \dots, a_r]^\top,
\end{split}
\label{eq:genlora-ba}
\end{equation}
resulting in the GenLoRA update:
\begin{equation}
\Delta W_{\text{Gen}} = 
\sum_{i=1}^{r} f^{B}_i(Z_B; \theta^{B}_i) \, f^{A}_i(Z_A; \theta^{A}_i)^\top.
\label{eq:genlora-dw}
\end{equation}
Under this parameterization, the number of trainable parameters is primarily determined by the sizes of $Z_B$, $Z_A$, and the nonlinear generators.
Specifically, the parameter complexity is $\mathcal{O}(m + n + r|\theta|)$, where $|\theta|$ denotes the number of parameters for a pair of $f_i^A$ and $f_i^B$.
Since $|\theta| \ll m, n$, this complexity is substantially lower than the $\mathcal{O}(r(m+n))$ parameter cost of standard LoRA.
By exploiting the expressive power of nonlinear mappings from latent vectors to basis vectors, GenLoRA significantly enhances adapter capacity, thereby realizing our core principle of Nonlinearity as Rank.

\subsection{RBF-based Functional Generators}
Radial basis functions (RBFs) are well known for their strong approximation capability and high parameter efficiency.
Motivated by these properties, we instantiate the nonlinear generators $f^B_i$ and $f^A_i$ using RBF-based parameterizations.To further achieve stable and expressive nonlinear mapping, we introduce instance-wise normalization and group-wise decomposition in the generator design, as shown in Figure~\ref{fig:fig3}(c).

\textbf{Group-wise decomposition.} Let $z \in \mathbb{R}^d$ denote an input latent vector. To avoid directly high-dimensional mapping, we partition $z$ into $G$ low-dimensional sub-vectors $\{\mathbf{x}_1, \dots, \mathbf{x}_G\}$, where each $\mathbf{x}_g \in \mathbb{R}^{d_g}$ such that $d = G \cdot d_g$. To ensure that the $r$ synthesized basis vectors remain distinct and avoid rank collapse, each matrix is generated using $r$ nonlinear generators, and each generator $f_i$ processes its corresponding sub-vectors using $G$ distinct transformations:
\begin{equation}
    f_i(z) = \text{concat} \left( \phi_{i,1}(\mathbf{x}_1), \dots, \phi_{i,G}(\mathbf{x}_G) \right) \in \mathbb{R}^d, 
    \label{eq:group_wise}
\end{equation}
where $\phi_{i,g}$ represents the specific transformation corresponding to the $g$-th sub-vector within the $i$-th generator. This decomposition reduces the mapping dimensionality of the nonlinear functions, effectively mitigating mapping complexity while enhancing their expressiveness. Furthermore, the independence of these sub-mappings enables parallel processing, significantly reducing computational overhead.

\paragraph{Instance-wise normalization.}
The learned nonlinear generators are designed to operate on inputs within a predefined interval, known as the grid range. However, since the numerical distribution of the input latent vectors is typically unknown and dynamic, manually adapting the fixed RBF grid to cover all input values is impractical. This mismatch often causes some input features to fall outside the effective receptive field of the grid basis functions, reducing approximation efficiency. To resolve this, instead of adjusting the grid, we normalize the input features to match a standard distribution. For the $g$-th sub-vector $\mathbf{x}_g$, we compute the normalized vector $\hat{\mathbf{x}}_g$:
\begin{equation}
\hat{\mathbf{x}}_g = \frac{\mathbf{x}_g - \mu_g}{\sigma_g + \varepsilon},
\label{eq:norm}
\end{equation}
where $\mu_g$ and $\sigma_g$ are the mean and standard deviation of $\mathbf{x}_g$ and $\varepsilon$ is for stability. We choose the fixed grid interval $[-3,3]$ as an empirical design, motivated by the approximately 99.7\% probability mass of a standard Gaussian within this range. Normalization does not imply Gaussian-distributed latent inputs, so this heuristic does not guarantee 99.7\% coverage; the actual fraction of inputs within the grid must be measured empirically.

Finally, the generator synthesizes values by combining a base non-linear activation with an RBF nonlinear expansion. Using the normalized vector $\hat{\mathbf{x}}_g$ from Eq.~\ref{eq:norm}, we define the transformation $\phi_{i,g}$ applied to the sub-vector $\mathbf{x}_g$ as:
\begin{equation}
    \phi_{i,g}(\mathbf{x}_g) = w_{b, i, g} \cdot \sigma(\hat{\mathbf{x}}_g) + \mathcal{F}_{\text{RBF}, i, g}(\hat{\mathbf{x}}_g),
    \label{eq:generator_output}
\end{equation}
where $\sigma(\cdot)$ is a standard activation function (e.g., SiLU), $w_{b, i, g}$ is a learnable scalar scaling weight specific to the $i$-th generator and the $g$-th group, and $\mathcal{F}_{\text{RBF}, i, g}(\cdot)$ represents the expressive RBF-based nonlinear term which will be detailed in Sec.~\ref{gaussian}.

\subsection{Gaussian Radial Basis Functions}
\label{gaussian}

For notational simplicity, we omit the generator index $i$ and group index $g$ below. Given the normalized input distribution, we define a grid of $K$ uniform centers $\{\mu_k\}_{k=1}^{K}$ over a fixed interval. The element-wise response of the $k$-th Gaussian basis function to the normalized input sub-vector $\hat{\mathbf{x}}_g$ is:
\begin{equation}
\varphi_k(\hat{\mathbf{x}}_g) = \exp\left( - \left( \frac{\hat{\mathbf{x}}_g - \mu_k}{h} \right)^2 \right),
\label{eq:gaussian_basis}
\end{equation}
where $h$ is the bandwidth determined by the grid spacing, and the operations are applied element-wise.

The final nonlinear output is a linear combination of these basis responses. Let $W_{\text{rbf}} = \{w_1, \dots, w_K\} \in \mathbb{R}^K$ denote the learnable weight vector associated with the current group. The function $\mathcal{F}_{\text{RBF}}(\hat{\mathbf{x}}_g)$ is defined as:
\begin{equation}
\mathcal{F}_{\text{RBF}}(\hat{\mathbf{x}}_g) = \sum_{k=1}^{K} w_k \cdot \varphi_k(\hat{\mathbf{x}}_g).
\label{eq:rbf_sum}
\end{equation}
In implementation, Eqs.~\eqref{eq:gaussian_basis} and \eqref{eq:rbf_sum} are vectorized via broadcasting, leveraging GPU parallelism for efficient inference.

\subsection{Theoretical Analysis}
\label{Theoretical Analysis}
We now analyze the theoretical properties of GenLoRA, focusing on the algebraic structure of the weight updates and optimization stability. These analyses establish preservation of the low-rank structure and conditional local gradient boundedness.

\begin{proposition}[Rank Boundedness]
\label{prop:rank_bound}
Let $\Delta W_{\text{Gen}} \in \mathbb{R}^{m \times n}$ be the weight update generated by GenLoRA with $r$ pairs of generator functions $\{f^B_i\}_{i=1}^r$ and $\{f^A_i\}_{i=1}^r$, where $r \ll \min(m, n)$. The rank of the update matrix satisfies:
\begin{equation}
\operatorname{rank}(\Delta W_{\text{Gen}}) \leq r.
\label{eq:rank_bound}
\end{equation}
\end{proposition}

Proposition~\ref{prop:rank_bound} guarantees that GenLoRA preserves the low-rank structure: each generated update has rank at most $r$, even though its factors depend nonlinearly on the latent inputs. The proof is provided in Appendix~\ref{app:proof_prop1}. A refined rank bound for the fixed-grid RBF parameterization is provided in Appendix~\ref{app:grouped-rank}.
\begin{proposition}[Local Gradient Boundedness]
\label{prop:grad_bound}
Fix $\varepsilon>0$ and $h>0$. On a compact neighborhood of the latent
vectors $Z$ and generator parameters $\Theta$, assume that
$\|\partial\mathcal{L}/\partial\Delta W_{\mathrm{Gen}}\|_F$ is bounded.
Then there exist finite constants $C_Z,C_\Theta>0$ such that
\begin{equation}
\left\|\frac{\partial\mathcal{L}}{\partial Z}\right\|_2\leq C_Z,
\qquad
\left\|\frac{\partial\mathcal{L}}{\partial\Theta}\right\|_2\leq C_\Theta.
\end{equation}
\end{proposition}

Under these local regularity conditions, stabilized normalization and
bounded Gaussian RBF derivatives yield a locally bounded generator
Jacobian. This includes zero-variance groups and does not require a
positive lower bound on their standard deviation. The proof and scope
of the result are provided in Appendix~\ref{app:proof_grad_bound}.

\begin{table*}[t]
\centering
\caption{The accuracy on Mathematical Reasoning tasks (Math10K) with various pretrained models.}
\label{tab:math_reasoning}
\resizebox{1\textwidth}{!}{
\begin{tabular}{l l *{8}{c}}
\toprule
& \textbf{{\textbf{Method}}} & \textbf{\#Params} & \textbf{AddSub} & \textbf{MultiArith} & \textbf{SingleEq} & \textbf{SVAMP} & \textbf{gsm8k} & \textbf{AQuA} & \textbf{Avg} \\
\midrule
\rowcolor{baselinecolor}
\cellcolor{white} & LoRA ($r=8$) & 4.72M & 82.28 & 85.83 & 89.76 & 66.40 & 56.18 & 23.23 & 67.28 \\
\rowcolor{baselinecolor}
\cellcolor{white} & MELoRA ($r=8$) & 4.72M & 83.80 & 84.00 & 86.61 & 62.90 & 47.84 & 22.83 & 64.66 \\
\rowcolor{baselinecolor}
\cellcolor{white} & HiRA ($r=8$) & 4.72M & 85.82 & 88.16 & 91.34 & 68.90 & 54.66 & 27.55 & 69.41 \\
\rowcolor{baselinecolor}
\cellcolor{white} & DoRA ($r=8$) & 4.92M & 83.04 & 87.67 & 91.14 & 66.60 & 56.41 & 21.65 & 67.75 \\
\rowcolor{aurora}
\cellcolor{white} & AuroRA ($r=2$) & 1.18M & 80.00 & 83.67 & 88.19 & 64.40 & 50.27 & 25.59 & 65.35 \\
\rowcolor{aurora}
\cellcolor{white} & AuroRA ($r=8$) & 4.72M & 80.25 & 87.50 & 89.57 & 67.30 & 56.86 & 20.47 & 66.99 \\
\rowcolor{ourcolor}
\cellcolor{white} & \textbf{GenLoRA ($r=8, g=16$)} & 0.98M & 88.35 & 90.67 & 93.70 & 70.80 & 53.90 & 23.62 & \underline{70.17} \\
\rowcolor{ourcolor}
\cellcolor{white} \multirow{-9}{*}{\rotatebox{90}{LLaMA-3-8B}} & \textbf{GenLoRA ($r=32, g=16$)} & 2.16M & 87.09 & 92.71 & 94.49 & 70.80 & 56.18 & 25.59 & \textbf{71.14} \\
\midrule
\rowcolor{baselinecolor}
\cellcolor{white} & LoRA ($r=8$) & 4.82M & 86.58 & 89.33 & 89.57 & 72.70 & 56.71 & 30.71 & 70.93 \\
\rowcolor{baselinecolor}
\cellcolor{white} & MELoRA ($r=8$) & 4.82M & 84.81 & 89.00 & 89.96 & 71.40 & 55.12 & 29.53 & 69.97 \\
\rowcolor{baselinecolor}
\cellcolor{white} & HiRA ($r=8$) & 4.82M & 85.85 & 88.67 & 90.55 & 72.10 & 55.04 & 26.38 & 69.76 \\
\rowcolor{baselinecolor}
\cellcolor{white} & DoRA ($r=8$) & 5.16M & 86.58 & 90.33 & 90.16 & 76.00 & 57.09 & 26.38 & \underline{71.09} \\
\rowcolor{aurora}
\cellcolor{white} & AuroRA ($r=2$) & 1.20M & 85.82 & 87.33 & 88.58 & 73.60 & 56.86 & 27.95 & 70.03 \\
\rowcolor{aurora}
\cellcolor{white} & AuroRA ($r=8$) & 4.82M & 86.84 & 89.83 & 90.75 & 73.60 & 57.70 & 26.38 & 70.85 \\
\rowcolor{ourcolor}
\cellcolor{white} & \textbf{GenLoRA ($r=8, g=16$)} & 0.95M & 84.56 & 91.33 & 90.94 & 71.00 & 56.86 & 28.74 & 70.57 \\
\rowcolor{ourcolor}
\cellcolor{white} \multirow{-9}{*}{\rotatebox{90}{Gemma-7B}} & \textbf{GenLoRA ($r=32, g=16$)} & 1.98M & 88.04 & 92.36 & 91.98 & 72.80 & 57.07 & 29.16 & \textbf{71.90} \\
\midrule
\rowcolor{baselinecolor}
\cellcolor{white} & LoRA ($r=8$) & 8.65M & 91.14 & 96.17 & 92.52 & 86.10 & 75.51 & 33.46 & 79.15 \\
\rowcolor{baselinecolor}
\cellcolor{white} & MELoRA ($r=8$) & 8.65M & 92.91 & 94.00 & 91.73 & 83.60 & 74.75 & 31.89 & 78.15 \\
\rowcolor{baselinecolor}
\cellcolor{white} & HiRA ($r=8$) & 8.65M & 92.41 & 94.67 & 91.14 & 84.40 & 75.21 & 35.43 & 78.88 \\
\rowcolor{baselinecolor}
\cellcolor{white} & DoRA ($r=8$) & 8.99M & 93.16 & 96.00 & 92.52 & 85.50 & 75.28 & 33.86 & 79.39 \\
\rowcolor{aurora}
\cellcolor{white} & AuroRA ($r=2$) & 2.16M & 93.41 & 95.00 & 91.73 & 84.10 & 73.24 & 33.07 & 78.43 \\
\rowcolor{aurora}
\cellcolor{white} & AuroRA ($r=8$) & 8.65M & 93.16 & 96.50 & 92.52 & 85.00 & 75.74 & 34.65 & 79.59 \\
\rowcolor{ourcolor}
\cellcolor{white} & \textbf{GenLoRA ($r=8, g=8$)} & 1.38M & 92.15 & 96.00 & 92.13 & 85.90 & 76.19 & 35.83 & \underline{79.70} \\
\rowcolor{ourcolor}
\cellcolor{white} \multirow{-9}{*}{\rotatebox{90}{Qwen2.5-14B}} & \textbf{GenLoRA ($r=32, g=8$)} & 2.26M & 91.65 & 96.17 & 93.31 & 87.80 & 74.68 & 38.58 & \textbf{80.36} \\
\bottomrule
\end{tabular}
}
\end{table*}

\section{Experiments}
In this section, we evaluate GenLoRA on three benchmarks across various architectures, followed by ablation studies verifying its scalability and visualizing singular values.

\subsection{Experimental Settings}
\label{Experimental Settings}
First,  we evaluate the natural language generation (NLG) capabilities of GenLoRA on the mathematical reasoning (Math10K)\citep{hu2023llm} and commonsense reasoning (Commonsense170K) benchmarks\citep{hu2023llm}. Both benchmarks comprise a training corpus and multiple test sub-tasks. For each benchmark, we fine-tune the models on the training data and subsequently assess performance across all sub-tasks. To demonstrate the versatility of GenLoRA, we conduct experiments utilizing various model architectures and scales, including Gemma-7B\citep{team2024gemma}, LLaMA-3-8B~\citep{dubey2024LLaMA}, and Qwen2.5-14B~\citep{DBLP:journals/corr/abs-2412-15115}.  Next, we assess the code generation capability of GenLoRA using LLaMA-3-8B~\citep{dubey2024LLaMA}.We fine-tune the model on the Magicoder-Evol-Instruct-110k dataset\citep{wei2023magicoder},  a curated and decontaminated subset of WizardCoder~\citep{luo2023wizardcoder} comprising high-quality instruction--response pairs for programming tasks. Subsequently, we assess the performance on the Humaneval+ test with 50 sampled completions per problem. We report Pass@1, Pass@5, and Pass@10 accuracy following the standard protocol via the BigCode Evaluation Harness~\citep{allal2022framework}. Further experimental details are provided in Appendix \ref{experiment details}.


\subsection{Overall Performance}
\paragraph{Results on Mathematical Reasoning Tasks.} 
As shown in Table~\ref{tab:math_reasoning}, GenLoRA demonstrates robust scalability and superior performance across diverse model architectures, consistently achieving higher accuracy with significantly fewer parameters. Specifically, on LLaMA3-8B, GenLoRA ($r=8$) outperforms LoRA ($r=8$) by 2.89 percentage points while utilizing only about one-fifth of the parameters (0.98M vs. 4.72M). Increasing GenLoRA to $r=32$ yields a 3.86-point gain over the same LoRA ($r=8$) baseline while still using less than half as many trainable parameters. On Gemma-7B, GenLoRA ($r=32,g=16$) reaches 71.90\% with 1.98M trainable parameters. On both Gemma-7B and Qwen2.5-14B, its higher-rank configuration surpasses all rank-8 baselines with fewer trainable parameters. In comparison, AuroRA, which aims to optimize the trade-off between parameter efficiency and model expressivity, falls short against GenLoRA. For instance, on the LLaMA3 benchmark, GenLoRA ($r=8$) not only surpasses AuroRA ($r=8$) by over 3\% in accuracy but does so using approximately 80\% fewer parameters. This clearly demonstrates that GenLoRA provides a far superior solution to the efficiency-performance trade-off.

\begin{table*}[t]
\centering
\caption{The accuracy on Commonsense Reasoning benchmarks with various pretrained models.}
\label{tab:commonsense}
\resizebox{1\textwidth}{!}{
\begin{tabular}{l l *{12}{c}}
\toprule
& \textbf{Method} & \textbf{\#Params} & \textbf{OBQA} & \textbf{ARC-C} & \textbf{Wino} & \textbf{PIQA} & \textbf{Social} & \textbf{ARC-E} & \textbf{BoolQ} & \textbf{Hella} & \textbf{Avg} \\
\midrule
\rowcolor{baselinecolor}
\cellcolor{white} & LoRA ($r=8$) & 4.72M & 81.60 & 78.50 & 82.08 & 87.54 & 77.79 & 91.67 & 70.40 & 87.95 & 82.19 \\
\rowcolor{baselinecolor}
\cellcolor{white} & MELoRA ($r=8$) & 4.72M & 74.80 & 75.00 & 71.43 & 84.71 & 72.98 & 89.44 & 67.74 & 80.84 & 77.12 \\
\rowcolor{baselinecolor}
\cellcolor{white} & HiRA ($r=8$) & 4.72M & 84.40 & 80.29 & 85.40 & 89.01 & 79.43 & 91.96 & 65.44 & 92.17 & 83.51 \\
\rowcolor{baselinecolor}
\cellcolor{white} & DoRA ($r=8$) & 4.92M & 81.00 & 78.41 & 82.24 & 87.65 & 77.99 & 91.58 & 70.15 & 88.12 & 82.14 \\
\rowcolor{aurora}
\cellcolor{white} & AuroRA ($r=2$) & 1.18M & 76.20 & 77.13 & 72.93 & 84.28 & 73.08 & 89.73 & 67.77 & 82.45 & 77.95 \\
\rowcolor{aurora}
\cellcolor{white} & AuroRA ($r=8$) & 4.72M & 81.60 & 78.50 & 82.08 & 87.27 & 77.48 & 91.62 & 70.46 & 88.05 & 82.13 \\
\rowcolor{ourcolor}
\cellcolor{white} & \textbf{GenLoRA ($r=8, g=16$)} & 0.98M & 85.60 & 79.01 & 84.85 & 88.57 & 79.38 & 91.75 & 71.83 & 91.45 & \underline{84.05} \\
\rowcolor{ourcolor}
\cellcolor{white} \multirow{-9}{*}{\rotatebox{90}{LLaMA-3-8B}} & \textbf{GenLoRA ($r=32, g=16$)} & 2.16M & 86.80 & 80.12 & 85.71 & 88.74 & 80.91 & 91.37 & 72.63 & 92.87 & \textbf{84.89} \\
\midrule
\rowcolor{baselinecolor}
\cellcolor{white} & LoRA ($r=8$) & 4.82M & 81.80 & 82.76 & 86.19 & 88.30 & 77.99 & 92.55 & 68.84 & 91.16 & 83.70 \\
\rowcolor{baselinecolor}
\cellcolor{white} & MELoRA ($r=8$) & 4.82M & 79.60 & 79.86 & 78.61 & 84.93 & 74.67 & 89.73 & 69.27 & 87.81 & 80.56 \\
\rowcolor{baselinecolor}
\cellcolor{white} & HiRA ($r=8$) & 4.82M & 85.80 & 83.45 & 85.00 & 88.30 & 77.94 & 93.14 & 69.66 & 92.31 & \underline{84.45} \\
\rowcolor{baselinecolor}
\cellcolor{white} & DoRA ($r=8$) & 5.16M & 82.60 & 80.89 & 84.06 & 85.47 & 75.95 & 91.16 & 69.20 & 89.42 & 82.34 \\
\rowcolor{aurora}
\cellcolor{white} & AuroRA ($r=2$) & 1.20M & 80.00 & 81.48 & 79.24 & 85.75 & 75.90 & 90.57 & 68.13 & 85.63 & 80.84 \\
\rowcolor{aurora}
\cellcolor{white} & AuroRA ($r=8$) & 4.82M & 82.40 & 82.76 & 83.50 & 85.96 & 76.82 & 91.75 & 70.28 & 90.08 & 82.94 \\
\rowcolor{ourcolor}
\cellcolor{white} & \textbf{GenLoRA ($r=8, g=8$)} & 0.77M & 84.80 & 80.80 & 86.90 & 87.54 & 77.43 & 92.26 & 71.25 & 92.07 & 84.13 \\
\rowcolor{ourcolor}
\cellcolor{white} \multirow{-9}{*}{\rotatebox{90}{Gemma-7B}} & \textbf{GenLoRA ($r=32, g=8$)} & 1.29M & 88.07 & 84.92 & 89.14 & 89.75 & 81.79 & 95.09 & 72.13 & 93.02 & \textbf{86.74} \\
\midrule
\rowcolor{baselinecolor}
\cellcolor{white} & LoRA ($r=8$) & 8.65M & 93.00 & 92.66 & 84.85 & 91.89 & 81.47 & 97.35 & 74.71 & 95.33 & 88.91 \\
\rowcolor{baselinecolor}
\cellcolor{white} & MELoRA ($r=8$) & 8.65M & 90.40 & 91.30 & 77.19 & 90.75 & 79.79 & 97.10 & 72.54 & 92.91 & 86.50 \\
\rowcolor{baselinecolor}
\cellcolor{white} & HiRA ($r=8$) & 8.65M & 94.20 & 93.26 & 89.98 & 93.09 & 83.88 & 97.98 & 75.87 & 96.50 & 90.60 \\
\rowcolor{baselinecolor}
\cellcolor{white} & DoRA ($r=8$) & 8.99M & 93.20 & 92.75 & 84.85 & 92.00 & 81.73 & 97.35 & 74.71 & 95.37 & 88.99 \\
\rowcolor{aurora}
\cellcolor{white} & AuroRA ($r=2$) & 2.16M & 90.60 & 91.21 & 78.22 & 90.86 & 80.19 & 97.05 & 73.49 & 93.81 & 86.93 \\
\rowcolor{aurora}
\cellcolor{white} & AuroRA ($r=8$) & 8.65M & 92.60 & 92.66 & 84.85 & 91.89 & 81.58 & 97.39 & 74.62 & 95.33 & 88.86 \\
\rowcolor{ourcolor}
\cellcolor{white} & \textbf{GenLoRA ($r=8, g=8$)} & 1.38M & 94.80 & 93.60 & 91.08 & 92.98 & 84.24 & 97.77 & 76.39 & 96.25 & \underline{90.89} \\
\rowcolor{ourcolor}
\cellcolor{white} \multirow{-9}{*}{\rotatebox{90}{Qwen2.5-14B}} & \textbf{GenLoRA ($r=32, g=8$)} & 2.26M & 94.80 & 94.20 & 90.77 & 93.31 & 84.44 & 98.11 & 76.79 & 96.36 & \textbf{91.10} \\
\bottomrule
\end{tabular}
}
\end{table*}

\begin{table*}[t!]
\centering
\caption{The performance on Code Generation tasks (HumanEval+) with LLaMA-3-8B fine-tuned on Magicoder-Evol-Instruct-110k.}
\label{tab:code_gen}
\resizebox{1\textwidth}{!}{
\begin{tabular}{l l *{6}{c}}
\toprule
& \textbf{Method} & \textbf{\#Params} & \textbf{Train Time} & \textbf{Relative Time} & \textbf{Pass@1} & \textbf{Pass@5} & \textbf{Pass@10} \\
\midrule
\rowcolor{baselinecolor}
\cellcolor{white} & LoRA ($r=8$) & 4.72M & 11.6h & $1.00\times$ & 20.79 & 35.48 & 41.31 \\
\rowcolor{baselinecolor}
\cellcolor{white} & MELoRA ($r=8$) & 4.72M & 10.3h & $0.88\times$ & 17.41 & 26.99 & 31.53 \\
\rowcolor{baselinecolor}
\cellcolor{white} & HiRA ($r=8$) & 4.72M & 10.6h & $0.91\times$ & 25.33 & 35.36 & 39.72 \\
\rowcolor{baselinecolor}
\cellcolor{white} & DoRA ($r=8$) & 4.92M & 11.9h & $1.03\times$ & 17.01 & 31.37 & 37.36 \\
\rowcolor{aurora}
\cellcolor{white} & AuroRA ($r=2$) & 1.18M & 11.1h & $0.96\times$ & 21.4 & 33.12 & 38.28 \\
\rowcolor{aurora}
\cellcolor{white} & AuroRA ($r=8$) & 4.72M & 11.8h & $1.02\times$ & 17.15 & 32.11 & 38.32 \\
\rowcolor{ourcolor}
\cellcolor{white} & \textbf{GenLoRA ($r=8, g=16$)} & 0.98M & 11.7h & $1.01\times$ & \underline{25.42} & \underline{36.36} & \underline{41.48} \\
\rowcolor{ourcolor}
\cellcolor{white} \multirow{-9}{*}{\rotatebox{90}{LLaMA-3-8B}} & \textbf{GenLoRA ($r=32, g=16$)} & 2.16M & 12.2h & $1.05\times$ & \textbf{28.30} & \textbf{41.92} & \textbf{47.77} \\
\bottomrule
\end{tabular}
}
\end{table*}
\paragraph{Results on Commonsense Reasoning Tasks.}
As shown in Table~\ref{tab:commonsense}, GenLoRA achieves the highest average accuracy among the evaluated methods in each backbone. At $r=8$, its gains over LoRA are 1.86, 0.43, and 1.98 percentage points on LLaMA-3-8B, Gemma-7B, and Qwen2.5-14B, respectively, averaging 1.42 points across the three models. Increasing the rank to 32 yields 84.89\%, 86.74\%, and 91.10\%, corresponding to gains of 2.70, 3.04, and 2.19 points over rank-8 LoRA. The Gemma-7B configuration uses $g=8$ and 1.29M trainable parameters. All three rank-32 configurations retain less than half the trainable parameter count of their corresponding rank-8 LoRA baselines. Comparisons at matched $r=32$ in Appendix~\ref{app:matched_rank_reasoning} further show that GenLoRA maintains competitive reasoning accuracy with \textbf{8.7\texttimes{}--15.3\texttimes{} fewer trainable parameters} than LoRA.

\paragraph{Results on Code Generation Tasks.} 
At matched rank, GenLoRA improves code-generation accuracy while using fewer trainable parameters. At $r=8$, Table~\ref{tab:code_gen} shows that GenLoRA reaches 25.42\% Pass@1 versus 20.79\% for LoRA, using 0.98M versus 4.72M parameters. At matched $r=32$, the comparison in Appendix~\ref{app:matched_rank_code} shows gains of 1.78, 1.73, and 2.08 percentage points in Pass@1, Pass@5, and Pass@10, respectively, with 2.16M versus 18.87M parameters. Separately, increasing GenLoRA to $r=32$ yields 28.30\% Pass@1, a 7.51-point gain over LoRA ($r=8$) while retaining less than half its parameter budget. This cross-rank comparison illustrates the benefit of expanding rank at low parameter cost; the matched-rank comparison isolates performance at the same rank.


\paragraph{Efficiency Analysis.} 
Beyond performance gains, we assess the computational overhead in Table 3. Remarkably, GenLoRA ($r=8$) achieves a training time of 11.7h, corresponding to a relative time of $1.01\times$, comparable to the efficiency of standard LoRA. We attribute this low latency primarily to the utilization of Radial Basis Functions. Compared to B-spline functions, which require complex knot management and can be less efficient for vectorized processing, the core RBF computation relies on pure tensor operations such as vectorized element-wise exponentiations and matrix multiplications. This structure is inherently suited for massive GPU parallelism, enabling GenLoRA to introduce powerful expressive capabilities with minimal computational overhead.

\subsection{Ablation Studies}

We evaluate instance-wise normalization and group-wise decomposition on
Math10K using LLaMA-3-8B, Gemma-7B, and Qwen2.5-14B.
Within each model, the full model and both ablated variants use the same
trainable parameter budget, including the variant without grouping.
As shown in Figure~\ref{fig:ablation_comparison}(a), removing normalization
reduces average accuracy by 6.96, 7.79, and 8.61 percentage points,
respectively. Removing grouping reduces accuracy by 4.12, 3.59, and
4.27 points. Both components therefore improve performance consistently
across the three backbones under matched parameter budgets.

For LLaMA-3-8B, all three variants use approximately 0.98M trainable
parameters. The ungrouped variant ($G=1$, $r=128$) matches the budget of
the grouped configuration ($G=16$, $r=8$) by increasing the number of
generators. However, under our fixed-grid RBF parameterization, the shared
normalized inputs and the $K$ RBF responses plus one base-activation
response imply $\operatorname{rank}(\Delta W)\le K+1=16$ for each
ungrouped update. Its lower accuracy shows that increasing the number of
generators does not recover grouped generation's performance in this
matched-budget setting. This compares parameterizations at the same
trainable parameter budget, rather than updates with actual ranks 128
and 8. Appendix~\ref{app:grouped-rank} provides the feature-space
factorization, rank bound, and implications for this ablation.

\begin{figure}[!t]
    \centering
    \begin{subfigure}[b]{0.49\linewidth}
        \centering
        \captionsetup{skip=1pt}
        \begin{minipage}[c][0.92\linewidth][c]{\linewidth}
            \centering
            \includegraphics[width=\linewidth,height=0.92\linewidth,keepaspectratio]{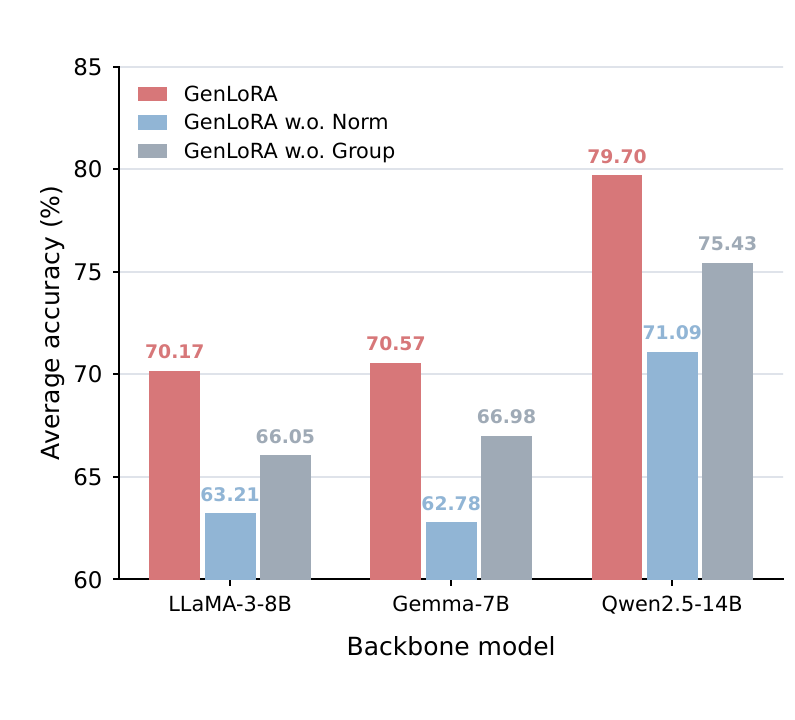}
        \end{minipage}
        \vspace{-8pt}
        \caption{Component ablations}
        \label{fig:component_ablation}
    \end{subfigure}\hfill
    \begin{subfigure}[b]{0.49\linewidth}
        \centering
        \captionsetup{skip=1pt}
        \begin{minipage}[c][0.92\linewidth][c]{\linewidth}
            \centering
            \includegraphics[width=\linewidth,height=0.92\linewidth,keepaspectratio]{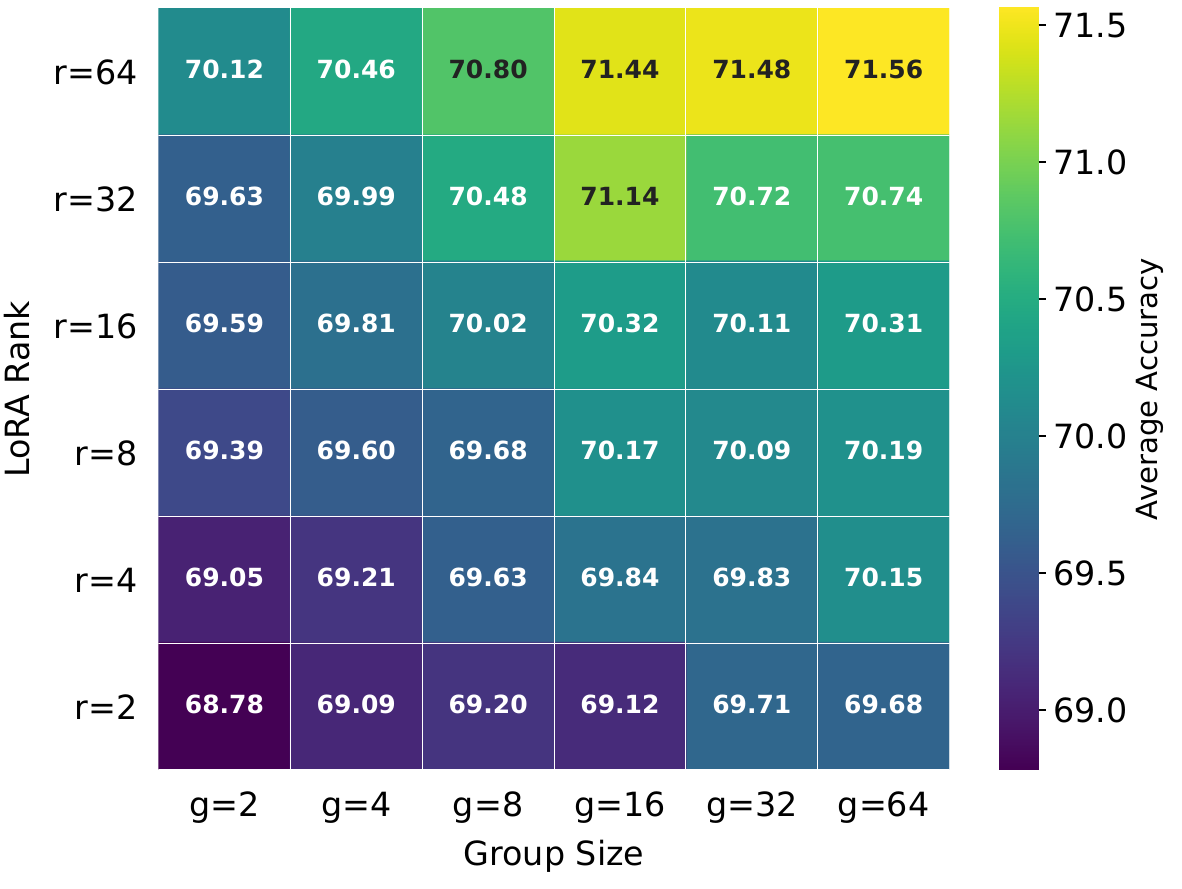}
        \end{minipage}
        \vspace{-8pt}
        \caption{Rank and group size}
        \label{fig:rank_group}
    \end{subfigure}
    \caption{Component and capacity ablations on Math10K. (a) Within each
    backbone, all three variants use the same trainable parameter budget;
    the vertical axis starts at 60\%. (b) Accuracy of GenLoRA across ranks
    and group sizes on LLaMA-3-8B.}
    \label{fig:ablation_comparison}
\end{figure}

\subsection{Scalability Analysis}
We evaluate the scalability of GenLoRA on mathematical reasoning tasks using LLaMA3-8B from three perspectives: the LoRA rank, the group size, and the tuning granularity.

\paragraph{Impact of Rank and Group Size.}
We investigate the interaction between rank ($r$) and group size ($g$) by varying both within $\{2, 4, 8, 16, 32, 64\}$, as visualized in Figure~\ref{fig:ablation_comparison}(b). Generally, increasing $r$ and $g$ enhances model capacity. However, we observe that the performance gains tend to stabilize around $g=16$; further increasing the group size does not yield significant improvements. We believe that this is because the fitting capability of the non-linear functions has approached its limit.
\paragraph{Impact of Tuning Granularity.}
Finally, we evaluate the scalability of GenLoRA under different tuning granularities. We introduce five distinct settings: Q, QV, QKV, QKVUD, and QKVOGUD, where Q, K, and V denote the query, key, and value projection weights, and O, G, U, and D represent the output, gate, up, and down projection weights, respectively. As shown in Figure~\ref{fig:genlora_analysis}(a), GenLoRA consistently outperforms LoRA across all configurations. Most notably, in the most restricted setting (Q), GenLoRA surpasses LoRA by 3.71 percentage points, highlighting its superior efficiency in feature extraction. With comprehensive tuning (QKVOGUD), GenLoRA scales effectively to a peak accuracy of 73.08\%, maintaining its robust lead over LoRA.

\begin{figure}[!t]
    \centering
    \begin{subfigure}[b]{0.49\linewidth}
        \centering
        \captionsetup{skip=1pt}
        \begin{minipage}[c][0.92\linewidth][c]{\linewidth}
            \centering
            \includegraphics[width=\linewidth,height=0.92\linewidth,keepaspectratio]{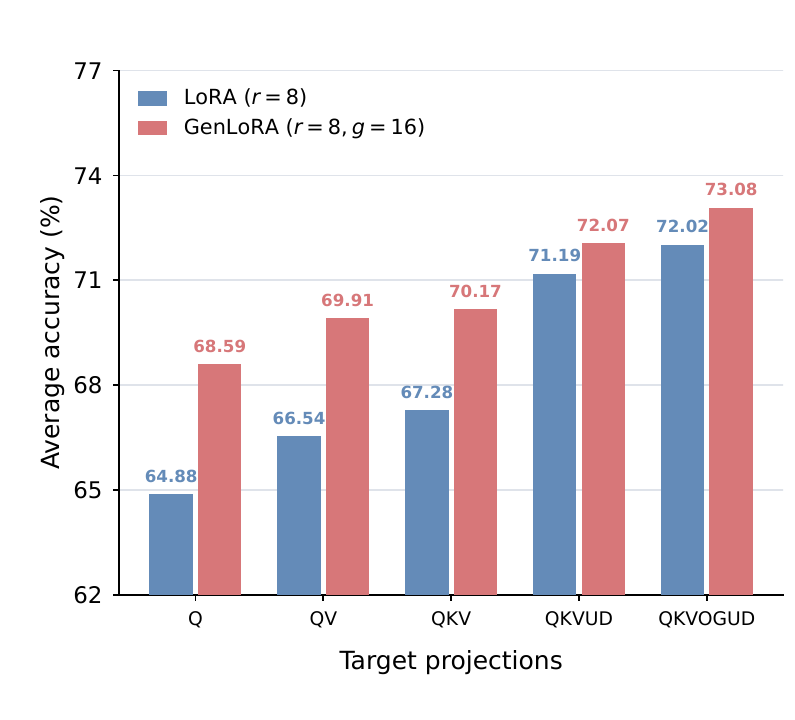}
        \end{minipage}
        \vspace{-8pt}
        \caption{Tuning granularity}
        \label{fig:tuning_granularity}
    \end{subfigure}\hfill
    \begin{subfigure}[b]{0.49\linewidth}
        \centering
        \captionsetup{skip=1pt}
        \begin{minipage}[c][0.92\linewidth][c]{\linewidth}
            \centering
            \includegraphics[width=\linewidth,height=0.92\linewidth,keepaspectratio]{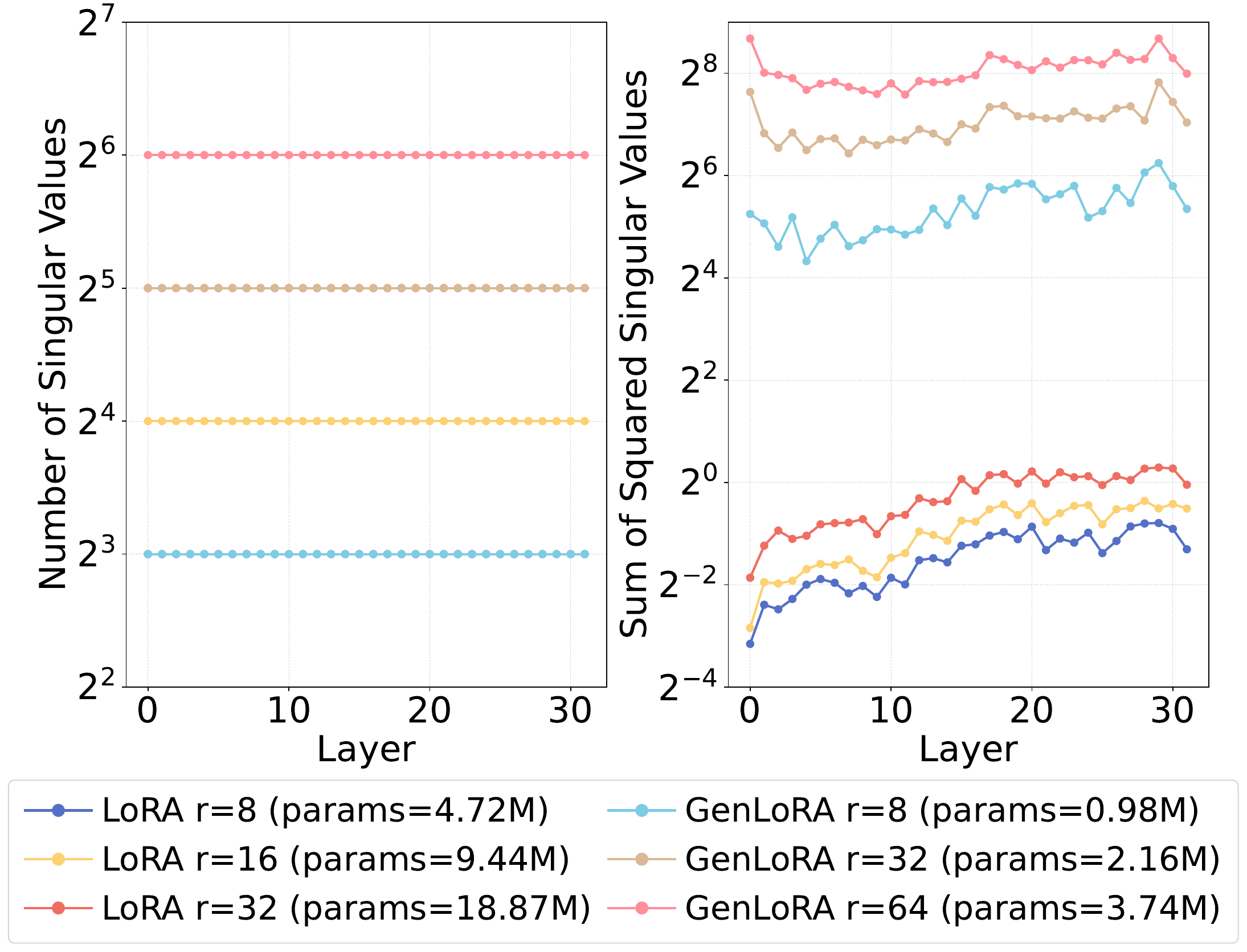}
        \end{minipage}
        \vspace{-8pt}
        \caption{Singular value analysis}
        \label{fig:singular_analysis}
    \end{subfigure}
    \caption{Adaptation scope and singular value analysis on LLaMA-3-8B.
    (a) Math10K accuracy across target-projection configurations; the vertical
    axis starts at 62\%. (b) Number and squared sum of singular values for
    query-projection updates. The singular-value counts of LoRA and GenLoRA
    at $r\in\{8,32\}$ overlap.}
    \label{fig:genlora_analysis}
\end{figure}

\subsection{Singular Value Analysis}

Figure~\ref{fig:genlora_analysis}(b) shows that both LoRA and GenLoRA attain numerical ranks matching their configured ranks. In particular, GenLoRA ($r=64$) attains a measured numerical rank of 64 using 3.74M parameters, an eightfold increase in measured numerical rank over LoRA ($r=8$, 4.72M).
The spectral energy satisfies
$\sum_i\sigma_i^2=\|\Delta W\|_F^2$ and therefore measures the squared
magnitude of the weight update. GenLoRA exhibits larger spectral energy
than LoRA across the evaluated layers; in particular, GenLoRA ($r=8$)
produces larger updates than LoRA ($r=32$) despite using fewer trainable
parameters. Larger energy alone does not establish greater expressiveness or better
generalization.

\section{Conclusion}
We introduce GenLoRA, which reduces the parameter cost of constructing low-rank basis vectors by generating them from shared latent vectors with compact nonlinear RBF generators. At the same rank, GenLoRA achieves competitive adaptation performance with substantially fewer trainable parameters. In reasoning experiments at matched $r=32$, it uses approximately 8.7\texttimes{}--15.3\texttimes{} fewer parameters than LoRA; in code generation at the same rank, it improves Pass@1 by 1.78 percentage points. Increasing rank while remaining below the parameter budget of lower-rank LoRA provides a complementary route to further performance gains. These results support generative basis-vector parameterization as an effective approach to parameter-efficient adaptation.

\bibliographystyle{iclr2027_conference} 
\bibliography{myrefs} 








\appendix

\section{Proof of Rank Boundedness}
\label{app:proof_prop1}

In this appendix, we provide the theoretical analysis for Proposition~\ref{prop:rank_bound}. We formally demonstrate that the algebraic structure of the GenLoRA update, constructed via generator functions, strictly adheres to the low-rank constraint.

\vspace{0.5em}
\noindent \textbf{Proposition~\ref{prop:rank_bound} (Restated).} \textit{
Let $\Delta W_{\text{Gen}} \in \mathbb{R}^{m \times n}$ be the weight update generated by GenLoRA with $r$ pairs of generator functions $\{f^B_i\}_{i=1}^r$ and $\{f^A_i\}_{i=1}^r$, where $r \ll \min(m, n)$. The rank of the update matrix satisfies:
\begin{equation}
\operatorname{rank}(\Delta W_{\text{Gen}}) \leq r.
\end{equation}
}

\vspace{0.5em}
First, we establish the algebraic form of the update. Given latent vectors $Z_B, Z_A$ and generator parameters $\Theta$, the instantiated basis vectors for $i \in \{1, \dots, r\}$ are $b_i = f^{B}_i(Z_B) \in \mathbb{R}^{m}$ and $a_i = f^{A}_i(Z_A) \in \mathbb{R}^{n}$. The GenLoRA update is defined as the summation of outer products: $\Delta W_{\text{Gen}} = \sum_{i=1}^{r} b_i a_i^\top$.

\vspace{0.5em}
\begin{lemma}
\label{A.1}
    The summation of outer products $\sum_{i=1}^{r} b_i a_i^\top$ is algebraically equivalent to the matrix product $BA$, where $B = [b_1, \dots, b_r]$ and $A = [a_1, \dots, a_r]^\top$.
\end{lemma}

\begin{proof}
We verify this equality using element-wise index notation. Let $(M)_{jk}$ denote the entry of a matrix $M$ at the $j$-th row and $k$-th column.
Compute the element $(BA)_{jk}$:
\begin{equation}
(BA)_{jk} = \sum_{p=1}^{r} (B)_{jp} (A)_{pk}.
\end{equation}
By construction, the $p$-th column of $B$ is $b_p$, so $(B)_{jp} = (b_p)_j$. The $p$-th row of $A$ is $a_p^\top$, so $(A)_{pk} = (a_p)_k$. Substituting these terms:
\begin{equation}
(BA)_{jk} = \sum_{p=1}^{r} (b_p)_j (a_p)_k.
\end{equation}
Now, compute the element of the summation term. The $(j, k)$-th entry of the outer product $b_i a_i^\top$ is $(b_i)_j (a_i)_k$. Summing over $i=1$ to $r$:
\begin{equation}
\left( \sum_{i=1}^{r} b_i a_i^\top \right)_{jk} = \sum_{i=1}^{r} (b_i)_j (a_i)_k.
\end{equation}
Since the element-wise expressions are identical for all $j, k$, it follows that $\Delta W_{\text{Gen}} = B A$.
\end{proof}

\vspace{0.5em}
\begin{lemma}
\label{A.2}
  The rank of the factorized matrix product $BA$ is bounded by $r$.  
\end{lemma}

\begin{proof}
Recall that the rank of a matrix corresponds to the dimension of its column space. The column space of the product $BA$, denoted $\mathcal{C}(BA)$, is defined as:
\begin{equation}
\mathcal{C}(BA) = \{ B(Av) \mid v \in \mathbb{R}^n \}.
\end{equation}
Since $Av$ is a vector in $\mathbb{R}^r$, let $u = Av$. Then $\mathcal{C}(BA) \subseteq \{ Bu \mid u \in \mathbb{R}^r \} = \mathcal{C}(B)$.
Consequently, the dimension of the column space of $BA$ cannot exceed the dimension of the column space of $B$:
\begin{equation}
\operatorname{rank}(BA) \leq \operatorname{rank}(B).
\end{equation}
The matrix $B \in \mathbb{R}^{m \times r}$ has exactly $r$ columns. The rank of a matrix is bounded by the number of its columns, so $\operatorname{rank}(B) \leq \min(m, r)$. Given the low-rank setting where $r < m$, we have $\operatorname{rank}(B) \leq r$.
Therefore, $\operatorname{rank}(\Delta W_{\text{Gen}}) = \operatorname{rank}(BA) \leq r$.
\end{proof}

\vspace{0.5em}
\noindent \textbf{Conclusion.}
From Lemma \ref{A.1}, the update $\Delta W_{\text{Gen}}$ is equivalent to the matrix product $BA$. From Lemma \ref{A.2}, the rank of this product is at most $r$. Thus, Proposition~\ref{prop:rank_bound} holds.

\paragraph{Scope of the rank guarantee.}
Proposition~\ref{prop:rank_bound} is a structural result: nonlinear basis
synthesis preserves the factorized update and its compatibility with weight
merging. It does not by itself establish improved approximation or accuracy.
Parameter efficiency follows instead from the incremental cost of lightweight
generators, as discussed in the method and scalability analysis. The RBF
approximation result in Appendix~\ref{app:rbf_theory} concerns a
sufficiently rich function family; it does not guarantee arbitrary basis
vectors with a fixed finite grid or fixed parameter budget. Likewise,
$\sum_j\sigma_j^2=\|\Delta W\|_F^2$ measures update magnitude, so increased
spectral energy alone does not prove greater expressiveness or better
generalization.

\subsection{Feature-space structure and rank bound of grouped generators}
\label{app:grouped-rank}

This section refines Proposition~\ref{prop:rank_bound} for the fixed-grid
RBF parameterization in Eqs.~\eqref{eq:group_wise}--\eqref{eq:rbf_sum}
and clarifies the interpretation of the grouping ablation.

\paragraph{Assumptions and notation.}
Within each factor, all generators use the same grouping, normalized
latent inputs, fixed RBF centers $\{\mu_k\}_{k=1}^{K}$, bandwidth
$h>0$, and base activation $\sigma$. The base-activation and RBF
combination coefficients are separately trainable for each generator
and group. The latent vectors may also be trainable.
We write
\begin{equation}
M_B=B\in\mathbb{R}^{m\times r},\qquad
M_A=A^\top\in\mathbb{R}^{n\times r},
\end{equation}
and let $d_B=m$ and $d_A=n$. For each side $s\in\{A,B\}$, the latent
vector is partitioned into $G_s$ groups, with group dimensions satisfying
$\sum_{g=1}^{G_s}d_g^s=d_s$. Each normalized group is viewed as a column
vector $\hat{x}_g^s\in\mathbb{R}^{d_g^s}$. The following analysis concerns
algebraic rank at any fixed parameter state.

\paragraph{Feature-space factorization.}
Define the Gaussian response functions
\begin{equation}
\phi_k(t)=\exp\!\left(-\frac{(t-\mu_k)^2}{h^2}\right),
\qquad k=1,\ldots,K.
\end{equation}
For group $g$ on side $s$, collect the shared responses into
\begin{equation}
\Psi_g^s=
\begin{bmatrix}
\sigma(\hat{x}_g^s)&\phi_1(\hat{x}_g^s)&\cdots&\phi_K(\hat{x}_g^s)
\end{bmatrix}
\in\mathbb{R}^{d_g^s\times(K+1)},
\end{equation}
where all functions are applied element-wise. The coefficient vector
of generator $i$ in this group is
\begin{equation}
c_{i,g}^s=
\begin{bmatrix}
w_{b,i,g}^s&w_{i,g,1}^s&\cdots&w_{i,g,K}^s
\end{bmatrix}^{\top}\in\mathbb{R}^{K+1}.
\end{equation}
Let
\begin{equation}
C_g^s=\begin{bmatrix}c_{1,g}^s&\cdots&c_{r,g}^s\end{bmatrix}
\in\mathbb{R}^{(K+1)\times r}.
\end{equation}
Eqs.~\eqref{eq:generator_output}--\eqref{eq:rbf_sum} then give
\begin{equation}
(M_s)_g=\Psi_g^s C_g^s,\qquad
M_s=\begin{bmatrix}
\Psi_1^s C_1^s\\\vdots\\\Psi_{G_s}^s C_{G_s}^s
\end{bmatrix}.
\end{equation}
Each column of $C_g^s$ retains the separately trainable coefficients of
one generator. Thus, this factorization does not identify the different
generators with a single function or tie their trainable coefficients.
It expresses their outputs as different linear combinations of the same
evaluated nonlinear responses within each group.

\paragraph{Rank bound.}
The rank of a vertically concatenated matrix is at most the sum of the
ranks of its blocks. Therefore,
\begin{align}
\operatorname{rank}(M_s)
&\le\min\left\{r,\sum_{g=1}^{G_s}
\operatorname{rank}(\Psi_g^s C_g^s)\right\}\\
&\le\min\left\{r,\sum_{g=1}^{G_s}
\operatorname{rank}(\Psi_g^s)\right\}\\
&\le\min\{r,d_s,G_s(K+1)\}.
\end{align}
Since $\operatorname{rank}(A)=\operatorname{rank}(A^\top)$ and
$\Delta W=BA$, it follows that
\begin{equation}
\operatorname{rank}(\Delta W)
\le\min\{r,m,n,G_A(K+1),G_B(K+1)\}.
\label{eq:grouped-rbf-rank-bound}
\end{equation}
This bound holds at every parameter state, including when the latent
vectors are optimized. Learning the latent vectors changes the evaluated
feature matrices $\Psi_g^s$, but does not change their number of columns.
The result is specific to the shared-input, fixed-grid parameterization
considered here, rather than a restriction on arbitrary nonlinear generators.

\paragraph{Implications for the grouping ablation.}
For the ungrouped configuration, $G_A=G_B=1$ and $K=15$, so
\begin{equation}
\operatorname{rank}(\Delta W)\le\min\{r,m,n,16\}.
\end{equation}
Consequently, $r=128$ specifies the number of generated columns and rows,
not an attainable update rank of 128. For fixed latent inputs, increasing
$r$ adds coefficient vectors but does not enlarge the shared feature
space of either factor.

Grouping instead introduces separately parameterized feature blocks on
disjoint coordinate subsets. With $G_A=G_B=16$ and $K=15$, the feature-count
ceiling is 256 per factor, before imposing the factor dimensions and the
nominal width $r$. In the grouped ablation configuration with $r=8$, the
update rank nevertheless remains at most 8.

The two configurations satisfy the same product $rG=128$, where
$G_A=G_B=G$. With both latent vectors trainable and no additional trainable
terms, their parameter counts for adapting a weight matrix are equal:
\begin{equation}
P=m+n+2rG(K+1).
\end{equation}
Their equal parameter counts do not imply identical feature spaces or
identical realizable updates.

The lower accuracy of the ungrouped variant supports the usefulness of
group-wise generation in this matched-budget setting. The factorization
above provides a structural distinction between the two parameterizations,
but the accuracy gap alone does not isolate its contribution from changes
in optimization. In particular, the result should not be interpreted as
showing that the grouped configuration has a higher actual update rank.

\paragraph{Scope for the main configurations.}
For $K=15$, $G_A,G_B\in\{8,16\}$, and $r\le64$, the feature-count terms
$G_A(K+1)$ and $G_B(K+1)$ are at least 128. Hence, this additional
feature-count ceiling does not tighten the usual bound
$\operatorname{rank}(\Delta W)\le r$ in these configurations. This
observation does not, by itself, guarantee that the learned factors attain
full rank or that the parameterization can represent every rank-$r$ update.

\section{Proof of Local Gradient Boundedness}
\label{app:proof_grad_bound}

We prove Proposition~\ref{prop:grad_bound} under explicit local regularity
conditions. Fix $\varepsilon>0$ and Gaussian bandwidth $h>0$, and restrict
the latent vectors and generator parameters to a compact neighborhood.
Assume that the upstream gradient
$\|\partial\mathcal{L}/\partial\Delta W_{\mathrm{Gen}}\|_F$ is uniformly
bounded there. This assumption is not implied merely by finite network
outputs, and compactness here does not assert that an entire training
trajectory remains bounded.

\paragraph{Stabilized normalization.}
For a group $z\in\mathbb{R}^d$, define
\begin{equation}
P=I-\frac{1}{d}\mathbf{1}\mathbf{1}^{T},\qquad
u=Pz,\qquad \sigma=\frac{\|u\|_2}{\sqrt d},\qquad
\hat z=\frac{u}{\sigma+\varepsilon}.
\end{equation}
The normalized input satisfies
$\|\hat z\|_2=\sqrt d\,\sigma/(\sigma+\varepsilon)\leq\sqrt d$.
For $\sigma>0$, its Jacobian is
\begin{equation}
J_{\mathrm{norm}}
=\frac{P}{\sigma+\varepsilon}
-\frac{uu^{T}}{d\sigma(\sigma+\varepsilon)^2}.
\end{equation}
Using $\|u\|_2^2=d\sigma^2$ gives
\begin{equation}
\|J_{\mathrm{norm}}\|_2
\leq\frac{1}{\sigma+\varepsilon}
+\frac{\sigma}{(\sigma+\varepsilon)^2}
\leq\frac{5}{4\varepsilon}.
\end{equation}
At $u=0$, for a perturbation $v$,
$\hat z(z+v)-Pv/\varepsilon=O(\|v\|_2^2)$, so the composite
normalization map is differentiable with Jacobian $P/\varepsilon$.
Thus no positive lower bound on the standard deviation is needed.

\paragraph{Generator and weight-update Jacobians.}
For $\varphi_k(t)=\exp(-(t-\mu_k)^2/h^2)$,
\begin{equation}
0<\varphi_k(t)\leq1,\qquad
|\varphi_k'(t)|\leq\frac{\sqrt{2/e}}{h}.
\end{equation}
Because normalized inputs are bounded, the SiLU residual and its derivative
are bounded on the relevant domain. The generator coefficients are bounded
on the compact parameter set. Consequently, the generated factors $A,B$
and their Jacobians with respect to $Z$ and $\Theta$ are uniformly bounded.
For the matrix product $\Delta W_{\mathrm{Gen}}=BA$,
\begin{equation}
\mathrm{d}\Delta W_{\mathrm{Gen}}
=(\mathrm{d}B)A+B(\mathrm{d}A),\qquad
\|\mathrm{d}\Delta W_{\mathrm{Gen}}\|_F
\leq\|A\|_2\|\mathrm{d}B\|_F+\|B\|_2\|\mathrm{d}A\|_F.
\end{equation}
This bounds the factor-to-update Jacobian that also enters the chain rule.
Combining these bounds with the assumed bound on the upstream gradient
yields finite constants $C_Z,C_\Theta>0$ such that
\begin{equation}
\left\|\frac{\partial\mathcal{L}}{\partial Z}\right\|_2\leq C_Z,
\qquad
\left\|\frac{\partial\mathcal{L}}{\partial\Theta}\right\|_2\leq C_\Theta,
\end{equation}
where the parameters are vectorized when taking norms.
The constants may depend on the neighborhood, model dimensions, $h$, and
$\varepsilon$. The result establishes conditional local boundedness, not
uniform control over arbitrary training trajectories or a guarantee of
stable optimization. In particular, small $\varepsilon$ or $h$ can still
produce large bounds.

\section{Algorithmic Complexity Analysis}
\label{app:complexity}

In this section, we provide a theoretical comparison between Radial Basis Functions (RBFs) and B-Spline from an algorithmic perspective. We analyze why RBFs constitute a more efficient choice for the high-dimensional generation tasks in GenLoRA.

\subsection{Computational Operations}

\paragraph{RBFs: Single-Step Global Mapping}
The RBF generator models nonlinearity via a unified, single-step calculation that is inherently parallelizable. For a given input $x$ and a grid resolution $N$ (denoting the number of basis functions), the algorithm executes three fundamental operations: distance calculation ($(x - \mu_k)^2$), Gaussian activation ($\exp(\cdot)$), and weighted aggregation. Crucially, these operations are strictly independent across the grid dimension $N$, meaning there are no data dependencies between the $k$-th and $(k+1)$-th basis functions. Consequently, the entire process can be fused and expressed as a single vectorized matrix operation. In terms of sequential complexity, this represents an $\mathcal{O}(1)$ depth algorithm, regardless of the grid size, making it exceptionally efficient on modern SIMD architectures like GPUs.

\paragraph{B-Spline: Recursive Dependency}
B-Splines are defined by the Cox-de Boor recursion algorithm, which inherently requires an iterative computation process. A spline of order $k$ (degree $k-1$) is computed by recursively evaluating basis functions from order $0$ up to $k$. This process begins with an initialization step to locate the input $x$ within the knot vector intervals. Subsequently, the algorithm enters a recursive loop of depth $k$, where at each depth $d \in [1, k]$, the value of a basis function is derived via linear interpolation of two basis functions from the previous depth $d-1$. While the total operation count scales as $\mathcal{O}(k^2)$, the primary computational bottleneck is the requirement for $\mathcal{O}(k)$ sequential iterations to resolve these dependencies. Unlike RBFs, B-Splines impose strict sequential constraints: the basis values at step $d$ cannot be computed until step $d-1$ is fully resolved, which significantly limits the potential for operator fusion and parallel execution.

\subsection{Space and Gradient Complexity}

In deep learning, the training memory cost is dominated by the requirement to store intermediate activations for gradient computation. For RBF generators, the computational graph is shallow, necessitating the backpropagation engine to trace through only a single exponential layer. Consequently, the space complexity remains linear with respect to the grid resolution $N$, i.e., $\mathcal{O}(N)$. In contrast, B-Spline generators construct an inherently deep computational graph due to their recursive definition. Since the algorithm iterates $k$ times, the automatic differentiation engine is forced to cache intermediate basis tensors for every step of the recursion to correctly apply the chain rule. This results in a space complexity of $\mathcal{O}(N \cdot k)$. For high-dimensional inputs, this multiplicative factor $k$ significantly amplifies the memory footprint during training, making B-Splines less memory-efficient than RBFs.

\subsection{Summary}

Table~\ref{tab:algo_comparison} summarizes the algorithmic differences. While B-Splines offer compact support, their recursive nature and higher memory complexity for gradient computation make them less suitable for the high-throughput requirements of GenLoRA. RBFs provide a mathematically simpler and parallel-friendly alternative that aligns better with the efficiency demands of large-scale neural network training.

\begin{table}[h]
\centering
\caption{Algorithmic comparison between RBF and B-Spline generators.}
\label{tab:algo_comparison}
\resizebox{0.55\columnwidth}{!}{
\begin{tabular}{l|cc}
\toprule
\textbf{Aspect} & \textbf{RBF (Ours)} & \textbf{B-Spline} \\
\midrule
\textbf{Structure} & Single-step Mapping & $k$-step Recursion \\
\textbf{Dependencies} & Independent (Parallel) & Sequential (Iterative) \\
\textbf{Gradient Memory} & $\mathcal{O}(N)$ & $\mathcal{O}(N \cdot k)$ \\
\bottomrule
\end{tabular}
}
\end{table}

\section{Approximation Capability of RBF Generators}
\label{app:rbf_theory}

In this section, we provide the theoretical justification for the strong fitting capability of our RBF-based generators, drawing upon the fundamental mathematical properties of radial basis functions analyzed by \citet{buhmann2000radial}.

\subsection{Universal Approximation Theorem}
\label{app:uat}

The theoretical foundation of our RBF generator is grounded in the analysis provided by \citet{park1991universal}. We first cite the original Theorem 2 regarding the approximation capabilities of RBF networks:

\begin{theorem}[Theorem 2 in \citet{park1991universal}]
\label{thm:uat_original}
Let $K: \mathbb{R}^r \to \mathbb{R}$ be an integrable bounded function such that $K$ is continuous almost everywhere and $\int_{\mathbb{R}^r} K(x)dx \neq 0$. Then the family $S_K$ is dense in $C(\mathbb{R}^r)$ with respect to the metric $d$ defined by
\begin{equation}
d(f,g) = \sum_{n=1}^{\infty} 2^{-n} \frac{\| (f-g) \cdot 1_{[-n,n]^r} \|_\infty}{1 + \| (f-g) \cdot 1_{[-n,n]^r} \|_\infty}.
\end{equation}
\end{theorem}

\noindent As further explained by the authors, the statement in Theorem 2 is equivalent to the statement that $S_K$ is uniformly dense on compacta in $C(\mathbb{R}^r)$. That is, for any continuous function $f$, any $\epsilon > 0$, and any compact subset $C \subset \mathbb{R}^r$, there exists a $q \in S_K$ such that $\| (q - f) \cdot 1_C \|_\infty < \epsilon$.

\subsection{Proof of Fitting Capability in GenLoRA}

We now connect the above theoretical result to the element-wise generator design in GenLoRA.

\paragraph{Setup.}
Since GenLoRA applies RBF generators in an element-wise manner for computational efficiency and parameter reduction, we model the target weight generation process at the scalar level. Specifically, we define a continuous 1D mapping $w^*: \mathbb{R} \to \mathbb{R}$ from a single scalar latent component $z^{(i)}$ to its corresponding target weight element.
We emphasize that the latent components $z^{(i)}$ are learnable parameters and are not explicitly constrained to lie in a compact domain. Accordingly, our analysis focuses on the approximation behavior of the 1D generator when restricted to any bounded interval of the latent space, which is the standard setting for applying universal approximation results.

The generator in GenLoRA employs 1D Gaussian radial basis functions of the form
\(
\phi(x) = \exp(-(x-\mu)^2/h^2),
\)
which are integrable, bounded, continuous, and have a non-zero integral over $\mathbb{R}$.
Therefore, the kernel $\phi$ satisfies the assumptions required in Theorem~\ref{thm:uat_original} for the univariate case ($r=1$).

\paragraph{Approximation Result.}
Let $\mathcal{Z}_R \subset \mathbb{R}$ denote an arbitrary compact interval of the scalar latent space, for example, $\mathcal{Z}_R = [-R, R]$.
Consider any continuous scalar-valued target function $w^*(z^{(i)})$ on this compact interval $\mathcal{Z}_R$. Under the assumption that the target weight elements can be independently modeled from their corresponding scalar latent inputs, by the uniform density property of Theorem~\ref{thm:uat_original}, for any $\epsilon > 0$, there exists a 1D RBF generator $g(z^{(i)}) \in S_K$ with sufficiently many basis functions such that
\begin{equation}
\sup_{z^{(i)} \in \mathcal{Z}_R} | g(z^{(i)}) - w^*(z^{(i)}) | < \epsilon.
\end{equation}

\paragraph{Conclusion.}
This result establishes the expressive power of the proposed RBF-based generator in a local sense: for any bounded interval of the latent space encountered during training, the element-wise generator can approximate the target 1D weight mapping to arbitrary precision given sufficient capacity. Combined with the locally bounded gradient analysis in Appendix~\ref{app:proof_grad_bound}, this characterizes its element-wise approximation capability and conditional local regularity, without guaranteeing global optimization stability.
\section{Algorithm Workflow}
\label{app:algo}

The complete training and inference workflow of GenLoRA is presented in Algorithm~\ref{alg:genlora}.

\begin{algorithm}[t!]
\caption{Algorithm Workflow of GenLoRA}
\label{alg:genlora}
\begin{algorithmic}[1]
\REQUIRE Pretrained weight $W_0\in\mathbb{R}^{m\times n}$; latents $Z_B,Z_A$; generator parameters $\Theta_B,\Theta_A$; rank $r$; group counts $G_B,G_A$; training data $\{(x_i,y_i)\}_{i=1}^N$.
\ENSURE Fine-tuned model weights $W$.

\STATE \textbf{Function} \textsc{Generator}($Z$, $\Theta$, $r$, $G$):
\STATE \quad Partition $Z$ into column vectors $\mathbf{x}_g\in\mathbb{R}^{d_g}$, $g=1,\ldots,G$;
\STATE \quad \textbf{for} $g=1,\ldots,G$ \textbf{do}
\STATE \quad \quad $\hat{\mathbf{x}}_g\leftarrow\text{InstanceNorm}(\mathbf{x}_g)$; \COMMENT{Eq.~\eqref{eq:norm}}
\STATE \quad \quad Allocate $V_g\in\mathbb{R}^{d_g\times r}$;
\STATE \quad \quad \textbf{for} $i=1,\ldots,r$ \textbf{do}
\STATE \quad \quad \quad $V_g[:,i]\leftarrow\phi_{i,g}(\mathbf{x}_g)$; \COMMENT{Eqs.~\eqref{eq:generator_output}--\eqref{eq:rbf_sum}}
\STATE \quad \quad \quad \COMMENT{Independent coefficients $w_{b,i,g}$ and $\{w_{i,g,k}\}_{k=1}^K$; shared $\hat{\mathbf{x}}_g$}
\STATE \quad \quad \textbf{end for}
\STATE \quad \textbf{end for}
\STATE \quad $M\leftarrow\operatorname{ConcatRows}(V_1,\ldots,V_G)\in\mathbb{R}^{(\sum_g d_g)\times r}$;
\STATE \quad \textbf{return} $M$;

\vspace{0.5em}
\STATE /* \textbf{Training Phase} */
\FOR{epoch $t \leftarrow 1$ to $T$}
    \FOR{each minibatch $\{x, y\}$ in training data}
        \STATE /* 1. Synthesize Rank Matrices via Generators */
        \STATE $B\leftarrow\textsc{Generator}(Z_B,\Theta_B,r,G_B)\in\mathbb{R}^{m\times r}$;
        \STATE $A\leftarrow[\textsc{Generator}(Z_A,\Theta_A,r,G_A)]^\top\in\mathbb{R}^{r\times n}$;
        
        \STATE /* 2. Forward pass with Adapter */
        \STATE $\Delta W \leftarrow B A$;
        \STATE $h \leftarrow W_0 x + \Delta W x$; \COMMENT{Eq.~\eqref{eq:lora}}
        
        \STATE Compute loss $\mathcal{L}(h, y)$;
        
        \STATE /* 3. Backpropagate through Generators and Latents */
        \STATE Update $\{Z_B, Z_A, \Theta_B, \Theta_A\}$ via gradient descent;
    \ENDFOR
\ENDFOR

\vspace{0.5em}
\STATE /* \textbf{Inference Preparation (Static Merge)} */
\STATE \textbf{Function} \textsc{MergeWeights}():
\STATE \quad $B_{\text{final}}\leftarrow\textsc{Generator}(Z_B,\Theta_B,r,G_B)$;
\STATE \quad $A_{\text{final}}\leftarrow[\textsc{Generator}(Z_A,\Theta_A,r,G_A)]^\top$;
\STATE \quad $\Delta W \leftarrow B_{\text{final}} A_{\text{final}}$; \COMMENT{Construct effective update}
\STATE \quad $W \leftarrow W_0 + \Delta W$; \COMMENT{Merge into backbone}
\STATE \quad \textbf{return} $W$;

\end{algorithmic}
\end{algorithm}
\paragraph{Zero-Latency Inference via Weight Merging.}
A critical advantage of GenLoRA is its compatibility with standard weight merging techniques. Unlike adapter methods that apply non-linearities directly to the input activation $x$ (e.g., $h = W_0x + B\sigma(Ax)$), GenLoRA restricts non-linearity solely to the parameter generation process.
Once training is complete, the latent vectors $Z$ and generator parameters $\Theta$ become static. Consequently, the synthesized matrices $B$ and $A$ evaluate to fixed constant matrices. This allows us to pre-compute the full update matrix $\Delta W = BA$ and merge it algebraically into the pretrained weights: $W_{\text{merged}} = W_0 + \Delta W$. During inference, the model utilizes $W_{\text{merged}}$ directly, incurring \textbf{zero additional computational overhead} compared to the base model.

\section{Standard Deviations}
\label{app:std_dev}

To ensure the reproducibility and reliability of our findings, we conducted primary experiments across three independent runs using different seeds. While the main body of the paper reports mean performance, we provide a detailed breakdown of the corresponding standard deviations in Table \ref{tab:std_reasoning}. It is observed that the performance fluctuations (standard deviations) for both mathematical and commonsense reasoning tasks are remarkably low. Crucially, the margin of improvement delivered by GenLoRA consistently outweighs these minor variabilities, confirming that our method's gains are statistically substantial rather than products of random noise.

\begin{table}[ht]
\centering
\caption{The standard deviation of the average accuracy on mathematical and commonsense reasoning tasks.}
\label{tab:std_reasoning}
\resizebox{\textwidth}{!}{
\begin{tabular}{lcccccc}
\toprule
\multirow{2}{*}{\textbf{Method}} & \multicolumn{3}{c}{\textbf{Mathematical Reasoning}} & \multicolumn{3}{c}{\textbf{Commonsense Reasoning}} \\
\cmidrule(lr){2-4} \cmidrule(lr){5-7}
 & \textbf{LLaMA-3-8B} & \textbf{Gemma-7B} & \textbf{Qwen2.5-14B} & \textbf{LLaMA-3-8B} & \textbf{Gemma-7B} & \textbf{Qwen2.5-14B} \\
\midrule
LoRA (r = 8) & 0.76 & 0.24 & 0.54 & 0.17 & 0.52 & 0.21 \\
MELoRA (r = 8) & 0.61 & 0.14 & 0.74 & 0.49 & 0.17 & 0.24 \\
HiRA (r = 8) & 0.38 & 0.45 & 0.22 & 0.21 & 0.35 & 0.11 \\
DoRA (r = 8) & 0.19 & 0.43 & 0.26 & 0.34 & 0.47 & 0.17 \\
AuroRA (r = 2) & 0.72 & 0.55 & 0.41 & 0.28 & 0.48 & 0.16 \\
AuroRA (r = 8) & 0.65 & 0.42 & 0.36 & 0.22 & 0.41 & 0.14 \\
\midrule
GenLoRA (r = 8) & 0.32 & 0.22 & 0.27 & 0.19 & 0.31 & 0.18 \\
GenLoRA (Higher Rank)$^\dagger$ & 0.41 & 0.29 & 0.38 & 0.28 & 0.37 & 0.16 \\
\bottomrule
\multicolumn{7}{l}{\small $^\dagger$Higher Rank denotes $r=32$ for all backbones.}
\end{tabular}
}
\end{table}

\section{Additional Comparative Experiments}
\label{app:additional_comparisons}

This appendix extends the evaluation of GenLoRA with additional PEFT baselines, comparisons of task-specific adaptation and cross-task generation, training-time and GPU-memory measurements, matched-rank reasoning experiments, and visual recognition results. We also discuss adapter storage and potential compatibility with advanced LoRA variants.

\subsection{Broader PEFT Paradigm Comparisons}
\label{app:broad_peft}

We first compare GenLoRA against a wide spectrum of PEFT baselines on the Math10K dataset using LLaMA-3-8B. The baselines are broadly categorized into three distinct paradigms:
\begin{itemize}
    \item \textbf{Prompt/Activation Tuning} (Prompt-Tuning\citep{lester2021power}, Prefix-Tuning\citep{li2021prefix}, $\text{IA}^3$\citep{liu2022few}): These methods are extremely parameter-efficient (typically $\sim 1$M parameters) but often incur additional inference overhead, such as extended context windows or latency costs, and generally struggle with complex reasoning tasks.
    \item \textbf{Spectral/Structural Tuning} (FourierFT\citep{gao2024parameter}, KronA\citep{edalati2025krona}): While these approaches outperform prompt tuning in mathematical reasoning, their rigid mathematical constraints inherently limit overall model expressivity.
    \item \textbf{Explicit Low-Rank/SVD} (LoRA, AdaLoRA\citep{zhang2023adalora}, PiSSA\citep{meng2024pissa}): These methods are highly robust and introduce zero additional inference latency via weight merging. However, they suffer from the explicit-rank bottleneck, where scaling the model capacity requires linear parameter growth relative to the model dimensions $m$ and $n$.
\end{itemize}

As demonstrated in Table~\ref{tab:broad_peft}, GenLoRA bridges the gap between these paradigms. By treating non-linearity as effective rank, GenLoRA's capacity scales with non-linear generators whose parameter size is strictly $\ll \min(m, n)$. Consequently, GenLoRA matches the minimal parameter footprint of prompt-tuning methods ($0.98$M) while significantly outperforming explicit-rank methods with an accuracy of $70.17\%$.

\begin{table}[!htbp]
    \centering
    \small
    \caption{Performance comparison across PEFT paradigms using LLaMA-3-8B
    fine-tuned on Math10K. Avg. is the unweighted mean across the six tasks.}
    \label{tab:broad_peft}
    \setlength{\tabcolsep}{4pt}
    \resizebox{\linewidth}{!}{%
    \begin{tabular}{lrrrrrrrr}
        \toprule
        Method & \#Params & AddSub & MultiArith & SingleEq & SVAMP & GSM8K & AQuA & Avg. \\
        \midrule
        Prompt-Tuning & 1.02M & 76.62 & 81.23 & 83.77 & 62.40 & 42.95 & 20.48 & 61.25 \\
        Prefix-Tuning & 1.31M & 80.49 & 83.64 & 85.76 & 64.50 & 45.09 & 23.41 & 63.82 \\
        FourierFT & 1.06M & 85.06 & 88.67 & 88.98 & 64.30 & 45.40 & 27.95 & 66.73 \\
        $\text{IA}^3$ & 1.25M & 86.84 & 83.50 & 86.42 & 64.50 & 50.64 & 26.38 & 66.38 \\
        LoRA ($r=8$) & 4.72M & 82.28 & 85.83 & 89.76 & 66.40 & 56.18 & 23.23 & 67.28 \\
        AdaLoRA ($r=8$) & 4.72M & 80.57 & 84.22 & 89.47 & 64.80 & 52.12 & 25.44 & 66.10 \\
        KronA & 1.52M & 80.76 & 84.17 & 89.37 & 63.80 & 50.11 & 27.95 & 66.03 \\
        PiSSA ($r=8$) & 4.72M & 82.78 & 86.89 & 91.27 & 66.40 & 54.11 & 25.33 & 67.80 \\
        \midrule
        \rowcolor{ourcolor}
        \textbf{GenLoRA ($r=8,g=16$)} & \textbf{0.98M} & 88.35 & 90.67 & 93.70 & 70.80 & 53.90 & 23.62 & \textbf{70.17} \\
        \bottomrule
    \end{tabular}%
    }
\end{table}

\subsection{Recent Nonlinear and Generative PEFT Baselines}
\label{app:recent_generative_peft}

We extend the Math10K evaluation on LLaMA-3-8B to NoRA and NoRA++,
MiSS~\citep{kang2026miss}, PEANuT~\citep{zhong2026peanut}, and WeGeFT~\citep{savadikar2023wegeft}, complementing the broader PEFT comparison in
Table~\ref{tab:broad_peft}. These baselines cover activation-space and
hybrid adaptation, parameter sharing, and weight generation.
A complementary comparison with LoRAGen~\citep{huang2026loragen}
in the cross-task generation setting is provided in
Appendix~\ref{app:cross_task_adaptation}.
Table~\ref{tab:recent_generative_peft} reports trainable parameter counts
and per-task accuracy across six mathematical reasoning benchmarks, together
with their unweighted average. With 0.98M parameters, GenLoRA achieves 70.17\% average accuracy,
exceeding rank-16 LoRA by 0.50 percentage points while using approximately
$9.6\times$ fewer parameters. It also outperforms the additional baselines
in this evaluation, supporting the accuracy--parameter trade-off of
compact basis-vector generation.

\begin{table}[!htbp]
    \centering
    \small
    \caption{Additional nonlinear and generative PEFT comparisons using
    LLaMA-3-8B fine-tuned on Math10K. Avg. is the unweighted mean accuracy
    (\%) across the six tasks.}
    \label{tab:recent_generative_peft}
    \setlength{\tabcolsep}{4pt}
    \resizebox{\linewidth}{!}{%
    \begin{tabular}{lrrrrrrrr}
        \toprule
        Method & \#Params & AddSub & MultiArith & SingleEq & SVAMP & GSM8K & AQuA & Avg. \\
        \midrule
        LoRA ($r=8$) & 4.72M & 82.28 & 85.83 & 89.76 & 66.40 & 56.18 & 23.23 & 67.28 \\
        LoRA ($r=16$) & 9.44M & 83.56 & 91.34 & 92.43 & 68.80 & 56.97 & 24.89 & 69.67 \\
        NoRA ($r=16$) & 1.06M & 81.76 & 82.26 & 85.14 & 64.70 & 53.19 & 19.07 & 64.36 \\
        NoRA++ ($r=16$) & 5.78M & 83.08 & 85.79 & 89.96 & 68.80 & 56.06 & 24.61 & 68.04 \\
        MiSS ($r=16$) & 3.15M & 84.67 & 86.93 & 88.71 & 64.60 & 55.21 & 25.35 & 67.58 \\
        PEANuT & 3.14M & 84.04 & 85.07 & 89.87 & 64.00 & 54.19 & 24.14 & 66.89 \\
        WeGeFT & 6.29M & 87.45 & 83.09 & 92.34 & 68.80 & 56.23 & 21.62 & 68.25 \\
        \midrule
        \rowcolor{ourcolor}
        \textbf{GenLoRA ($r=8,g=16$)} & \textbf{0.98M} & 88.35 & 90.67 & 93.70 & 70.80 & 53.90 & 23.62 & \textbf{70.17} \\
        \bottomrule
    \end{tabular}%
    }
\end{table}

\subsection{Cross-Task Generation versus Task-Specific Adaptation}
\label{app:cross_task_adaptation}

GenLoRA is designed for parameter-efficient task-specific adaptation,
whereas LoRAGen~\citep{huang2026loragen} amortizes adaptation through a shared generator trained
across source tasks. Text-to-LoRA and Doc-to-LoRA likewise motivate the
distinction between conditional adapter generation and direct target-task
optimization; the quantitative comparison below evaluates LoRAGen.
We follow LoRAGen's unseen-task protocol: its generator
is trained on 136 FLAN source tasks with 20 natural-language descriptions
per task, excluding AP-Rec, AP-Pos, QASC-1, QASC-2, WQA-T, and WQA-A.
At evaluation, LoRAGen receives only task descriptions; we generate three
adapters per task and average their scores. LoRA and GenLoRA instead
use the corresponding target-task training data. All adapted methods
use FLAN-T5-Large, rank $r=16$, and the same six evaluation tasks and
pipeline.

\begin{table}[!htbp]
    \centering
    \small
    \caption{Cross-task generation and task-specific adaptation on
    FLAN-T5-Large. Parameter counts and training times for LoRA and GenLoRA
    are totals across the six target tasks; LoRAGen's entries refer to its
    shared generator and exclude source-task adapter-library construction.
    LoRAGen uses no target-task training data, whereas LoRA and GenLoRA do.}
    \label{tab:cross_task_adaptation}
    \resizebox{\linewidth}{!}{%
    \begin{tabular}{lrrrrrrrrr}
        \toprule
        Method & \#Params & Time (h) & AP-Rec & AP-Pos & QASC-1 & QASC-2 & WQA-T & WQA-A & Avg. \\
        \midrule
        FLAN-T5-Large (unadapted) & 0 & 0 & 69.5 & 35.1 & 24.5 & 10.2 & 9.6 & 62.4 & 35.22 \\
        LoRAGen ($r=16$) & 54.33M & 3.90 & 75.4 & 42.2 & 29.6 & 15.2 & 14.9 & 68.4 & 40.95 \\
        LoRA ($r=16$) & 4.72M & 0.32 & 80.1 & 45.3 & 68.1 & 56.7 & 75.2 & 72.9 & 66.38 \\
        \rowcolor{ourcolor}
        GenLoRA ($r=16,g=16$) & 1.33M & 0.32 & 82.3 & 49.2 & 67.4 & 57.2 & 76.8 & 74.3 & \textbf{67.87} \\
        \bottomrule
    \end{tabular}%
    }
\end{table}

As shown in Table~\ref{tab:cross_task_adaptation}, LoRAGen improves over
unadapted FLAN-T5-Large by 5.73 points without target-task training.
With target-task supervision, GenLoRA reaches 67.87\%, exceeding LoRA
by 1.49 points with $3.55\times$ fewer stored adapter parameters and the
same reported total training time. Its 26.92-point advantage over LoRAGen
reflects a different adaptation regime, not an isolated comparison of
generator architectures. LoRAGen trades an upfront training cost for
adaptation without target-task optimization; GenLoRA is suitable when
lightweight target-task training is available. The reported times describe
these particular workloads and do not establish a universal runtime
advantage as the number of tasks grows.

\paragraph{Storage across tasks.}
Using the aggregate counts in Table~\ref{tab:cross_task_adaptation}, the
average GenLoRA adapter contains $1.33/6\approx0.222$M parameters.
For comparable target-task configurations, retaining $T$ adapters requires
\begin{equation}
S_{\mathrm{GenLoRA}}(T)\approx\frac{1.33}{6}T\ \mathrm{M},
\qquad
T^\star\approx\frac{54.33}{1.33/6}\approx245.
\end{equation}
Thus the adapter collection remains smaller than the 54.33M-parameter
shared generator below approximately 245 tasks under this accounting.
This estimate assumes equal storage precision, excludes the shared backbone,
optimizer states, metadata, and LoRAGen's source-task library, and assumes
that generated LoRAGen adapters are not cached. It is a configuration-specific
parameter-count comparison rather than a general storage bound.
GenLoRA adapters can be stored and loaded on demand using an adapter-library
workflow, or merged into a deployment model without additional inference
computation.

\subsection{Training Time and GPU Memory Analysis}
\label{app:matched_rank_code}

We compare LoRA and GenLoRA at the same rank using the code-generation
setting of Table~\ref{tab:code_gen}: LLaMA-3-8B fine-tuned on
Magicoder-Evol-Instruct-110k and evaluated on HumanEval+.
Table~\ref{tab:matched_rank_code} reports training time, peak GPU memory,
and accuracy. At $r=8$, GenLoRA adds approximately 1\% training time and
reduces peak memory by 1,264 MB (4.0\%). At $r=32$, both methods take
12.2 hours at the reported precision; GenLoRA uses $8.7\times$ fewer
trainable parameters, saves 223 MB of peak memory, and improves
Pass@1/5/10 by 1.78/1.73/2.08 percentage points.
The $1.05\times$ relative time is measured against LoRA at $r=8$,
not against a matched-rank baseline. Thus, the measured generator overhead
is small in these settings, while peak-memory savings depend on the workload.

\begin{table}[!htbp]
    \centering
    \small
    \caption{Matched-rank code-generation results. Relative training time
    uses LoRA ($r=8$) as the reference for all rows.}
    \label{tab:matched_rank_code}
    \resizebox{\linewidth}{!}{%
    \begin{tabular}{lrrrrrrr}
        \toprule
        Method & \#Params & Time (h) & Relative & Peak GPU (MB) & Pass@1 & Pass@5 & Pass@10 \\
        \midrule
        LoRA ($r=8$) & 4.72M & 11.6 & $1.00\times$ & 31,724 & 20.79 & 35.48 & 41.31 \\
        \rowcolor{ourcolor}
        GenLoRA ($r=8,g=16$) & 0.98M & 11.7 & $1.01\times$ & 30,460 & 25.42 & 36.36 & 41.48 \\
        LoRA ($r=32$) & 18.87M & 12.2 & $1.05\times$ & 32,095 & 26.52 & 40.19 & 45.69 \\
        \rowcolor{ourcolor}
        GenLoRA ($r=32,g=16$) & 2.16M & 12.2 & $1.05\times$ & 31,872 & \textbf{28.30} & \textbf{41.92} & \textbf{47.77} \\
        \bottomrule
    \end{tabular}%
    }
\end{table}

\subsection{Matched-Rank Comparisons on Reasoning Tasks}
\label{app:matched_rank_reasoning}

To separate the effect of rank from that of parameterization, we evaluate
LoRA, MELoRA, HiRA, DoRA, AuroRA, and GenLoRA at matched ranks under
the same training and evaluation settings for each model and task family.
We use $r=32$ for all three backbones.
Table~\ref{tab:matched_rank_reasoning} reports average accuracy on the
mathematical and commonsense reasoning benchmarks.
GenLoRA attains the highest accuracy in four of the six settings and
trails the strongest baseline by only 0.06 and 0.13 percentage points
in the two Qwen2.5-14B settings. Its trainable parameter count is
$8.7\times$--$15.3\times$ smaller than matched-rank LoRA under the stated QKV configuration.
These results show that the parameter-efficiency advantage persists
when the baselines use the same rank.

\begin{table}[!htbp]
    \centering
    \small
    \caption{Matched-rank reasoning comparisons. Parameter counts and
    average accuracies are reported separately for each task family.
    GenLoRA uses $g=16$ for LLaMA-3-8B and $g=8$ for Qwen2.5-14B;
    for Gemma-7B it uses $g=16$ on Math10K and $g=8$ on commonsense reasoning.
    Bold accuracies mark the best result within each model and task family.
}
    \label{tab:matched_rank_reasoning}
    \setlength{\tabcolsep}{5pt}
    \begin{tabular}{llrrrr}
        \toprule
        & & \multicolumn{2}{c}{Math10K} & \multicolumn{2}{c}{Commonsense} \\
        \cmidrule(lr){3-4}\cmidrule(lr){5-6}
        Model / Rank & Method & \#Params & Avg. & \#Params & Avg. \\
        \midrule
        LLaMA-3-8B & LoRA & 18.87M & 70.58 & 18.87M & 84.77 \\
        ($r=32$) & MELoRA & 18.87M & 69.16 & 18.87M & 82.16 \\
        & HiRA & 18.87M & 70.97 & 18.87M & 83.93 \\
        & DoRA & 19.07M & 70.75 & 19.07M & 84.58 \\
        & AuroRA & 18.87M & 70.42 & 18.87M & 83.73 \\
        \rowcolor{ourcolor}
        & GenLoRA & 2.16M & \textbf{71.14} & 2.16M & \textbf{84.89} \\
        \midrule
        Gemma-7B & LoRA & 19.27M & 71.79 & 19.27M & 86.14 \\
        ($r=32$) & MELoRA & 19.27M & 71.05 & 19.27M & 83.67 \\
        & HiRA & 19.27M & 71.64 & 19.27M & 86.16 \\
        & DoRA & 19.61M & 71.52 & 19.61M & 85.15 \\
        & AuroRA & 19.27M & 71.23 & 19.27M & 84.78 \\
        \rowcolor{ourcolor}
        & GenLoRA & 1.98M & \textbf{71.90} & 1.29M & \textbf{86.74} \\
        \midrule
        Qwen2.5-14B & LoRA & 34.60M & 80.09 & 34.60M & 90.68 \\
        ($r=32$) & MELoRA & 34.60M & 79.45 & 34.60M & 89.75 \\
        & HiRA & 34.60M & 79.78 & 34.60M & 90.89 \\
        & DoRA & 34.95M & \textbf{80.42} & 34.95M & \textbf{91.23} \\
        & AuroRA & 34.60M & 80.17 & 34.60M & 90.77 \\
        \rowcolor{ourcolor}
        & GenLoRA & 2.26M & 80.36 & 2.26M & 91.10 \\
        \bottomrule
    \end{tabular}
\end{table}

\subsection{Compatibility with Advanced LoRA Variants}
\label{app:variant_compatibility}

GenLoRA changes how the factors are parameterized while retaining the
update structure $\Delta W_{\mathrm{Gen}}=BA$. This permits potential
combinations with methods that modify weight decomposition, rank allocation,
or initialization. The constructions below describe structural compatibility;
we do not report combined-method experiments here.

\paragraph{Magnitude--direction decomposition.}
In DoRA~\citep{liu2024dora}, generated factors can parameterize the
low-rank directional update:
\begin{equation}
W'=m\odot\frac{W_0+\Delta W_{\mathrm{Gen}}}
{\|W_0+\Delta W_{\mathrm{Gen}}\|_c},
\qquad \Delta W_{\mathrm{Gen}}=B_\Theta(Z)A_\Theta(Z),
\end{equation}
where $\|\cdot\|_c$ denotes column-wise norms and the learnable magnitude
vector $m$ is broadcast over columns. For nonzero column norms, this
preserves DoRA's magnitude--direction decomposition while replacing explicit
low-rank factors with compact generators.

\paragraph{Adaptive rank allocation.}
For AdaLoRA~\citep{zhang2023adalora}, the generated update can be expressed
as $\sum_{i=1}^r b_i a_i^\top$, where each pair of generators supplies one
rank component. Importance-based selection can in principle retain or prune
these components across layers. A concrete integration must adapt the
importance scores, singular-value gates, and orthogonality regularization
to the generated factors; retaining the factorized form alone does not
make the original allocation algorithm a drop-in replacement.

\paragraph{Task-aware initialization and allocation.}
Task-aware strategies such as TLoRA could determine layer-wise rank budgets,
scaling, and target initial subspaces, while GenLoRA generates the associated
basis vectors. Matching prescribed initial factors may require fitting the
generators within their representational capacity. Such combinations offer
an extension of the parameterization, but their optimization behavior and
empirical benefits require separate evaluation.

\subsection{Comparison with Basis-Combination Methods}
\label{app:nola_comparison}

NOLA~\citep{koohpayegani2023nola} and HyperLoRA~\citep{li2025hyperlora}
construct updates through basis combinations, but differ in how the bases
are parameterized: NOLA uses fixed random bases, whereas HyperLoRA uses
trainable bases. GenLoRA instead synthesizes low-rank basis vectors from
shared latent vectors using nonlinear RBF generators.

Table~\ref{tab:nola_comparison} compares GenLoRA with NOLA under comparable
parameter budgets. GenLoRA achieves 70.17\% average accuracy versus
65.15\% for NOLA. This result supports the effectiveness of the proposed
parameterization in the evaluated setting; it is not a direct empirical
comparison with HyperLoRA or a proof of greater expressiveness.

\begin{table}[!htbp]
    \centering
    \small
    \caption{Comparison with NOLA under comparable parameter budgets using
    LLaMA-3-8B fine-tuned on Math10K. Avg. is the unweighted mean accuracy
    (\%) across the six tasks.}
    \label{tab:nola_comparison}
    \setlength{\tabcolsep}{4pt}
    \resizebox{\linewidth}{!}{%
    \begin{tabular}{lrrrrrrrr}
        \toprule
        Method & \#Params & AddSub & MultiArith & SingleEq & SVAMP & GSM8K & AQuA & Avg. \\
        \midrule
        NOLA & 1.04M & 82.73 & 83.34 & 85.48 & 65.80 & 50.70 & 22.83 & 65.15 \\
        \rowcolor{ourcolor}
        \textbf{GenLoRA ($r=8,g=16$)} & \textbf{0.98M} & 88.35 & 90.67 & 93.70 & 70.80 & 53.90 & 23.62 & \textbf{70.17} \\
        \bottomrule
    \end{tabular}%
    }
\end{table}

\subsection{Generalization to Visual Modalities and Discriminative Architectures}
\label{app:visual_generalization}

To evaluate the generalizability of GenLoRA beyond text-based generative autoregressive language models, we extended our experiments to the visual modality and discriminative classification tasks. Specifically, we applied GenLoRA to encoder-only Vision Transformer (ViT) architectures (ViT-Base and ViT-Large) and evaluated them across eight standard image classification datasets. 

As shown in Table~\ref{tab:vit_classification}, GenLoRA exhibits strong scalability and parameter efficiency on visual tasks. At an equivalent rank of $r=8$, GenLoRA matches the performance of standard LoRA while utilizing approximately $25\%$ of the parameters. More importantly, when allocated comparable parameter budgets (e.g., scaling GenLoRA to $r=32$ for ViT-Base and $r=64$ for ViT-Large), our method achieves significant performance improvements of 1.02 and 1.46 percentage points in average accuracy, respectively. These results firmly establish that the structural advantages of our generative adaptation approach are highly modality-agnostic, generalizing exceptionally well from natural language domains to diverse visual modalities and encoder-only discriminative architectures.

\begin{table*}[h]
    \centering
    \caption{Evaluation of GenLoRA on visual modalities and discriminative models (Vision Transformers) across 8 standard image classification datasets.}
    \label{tab:vit_classification}
    \vspace{0.2cm}
    \resizebox{\textwidth}{!}{
    \begin{tabular}{llcccccccccc}
        \toprule
        \textbf{Model} & \textbf{Method} & \textbf{\#Params} & \textbf{OxfordPet} & \textbf{StanfordC} & \textbf{CIFAR10} & \textbf{DTD} & \textbf{EuroSAT} & \textbf{FGVC} & \textbf{RESISC45} & \textbf{CIFAR100} & \textbf{Avg} \\
        \midrule
        \multirow{3}{*}{ViT-Base} 
        & LoRA ($r=8$) & 0.29M & 93.11 & 74.12 & 98.21 & 77.85 & 98.60 & 50.36 & 93.55 & 91.12 & 84.62 \\
        & GenLoRA ($r=8, g=8$) & 0.09M & 93.34 & 74.96 & 98.25 & 76.99 & 98.35 & 50.54 & 93.81 & 91.21 & 84.68 \\
        & GenLoRA ($r=32, g=8$) & 0.23M & 93.50 & 77.24 & 98.80 & 77.44 & 98.80 & 52.66 & 94.30 & 92.40 & \textbf{85.64} \\
        \midrule
        \multirow{3}{*}{ViT-Large} 
        & LoRA ($r=8$) & 0.79M & 94.12 & 84.40 & 98.50 & 77.93 & 98.60 & 58.42 & 93.68 & 93.25 & 87.36 \\
        & GenLoRA ($r=8, g=8$) & 0.20M & 93.87 & 84.58 & 98.64 & 79.34 & 98.63 & 57.89 & 93.84 & 93.89 & 87.59 \\
        & GenLoRA ($r=64, g=8$) & 0.88M & 94.88 & 85.99 & 99.10 & 81.45 & 98.94 & 59.76 & 95.43 & 94.98 & \textbf{88.82} \\
        \bottomrule
    \end{tabular}
    }
\end{table*}

\section{Additional Ablation Studies}
\label{app:more_ablations}

In this section, we present additional micro-level ablation studies on the Math10K dataset using LLaMA-3-8B with GenLoRA ($r=8, g=16$) to validate our core architectural choices and hyperparameter settings.

\subsection{Ablation on Hyperparameter $K$ (RBF Centers)}
\label{app:ablation_k}

We first evaluate the impact of the number of RBF centers, denoted as $K$. As shown in Table~\ref{tab:ablation_k}, $K=15$ serves as a robust ``sweet spot'' for our architecture. While a lower value of $K$ (e.g., $K=5$ or $10$) limits the mapping resolution of the generator, increasing $K$ beyond $15$ introduces additional parameters without yielding further performance gains, leading to diminishing returns. 

\begin{table}[!htbp]
    \centering
    \small
    \caption{Ablation on the number of RBF centers ($K$) for GenLoRA
    ($r=8,g=16$) using LLaMA-3-8B fine-tuned on Math10K.
    Avg. is the unweighted mean accuracy (\%) across the six tasks.}
    \label{tab:ablation_k}
    \setlength{\tabcolsep}{4pt}
    \resizebox{\linewidth}{!}{%
    \begin{tabular}{rrrrrrrrr}
        \toprule
        $K$ & \#Params & AddSub & MultiArith & SingleEq & SVAMP & GSM8K & AQuA & Avg. \\
        \midrule
        5 & 0.74M & 73.42 & 80.83 & 86.22 & 56.90 & 48.83 & 20.47 & 61.11 \\
        10 & 0.86M & 78.70 & 83.73 & 89.83 & 63.80 & 50.78 & 23.87 & 65.12 \\
        15 & 0.98M & 88.35 & 90.67 & 93.70 & 70.80 & 53.90 & 23.62 & \textbf{70.17} \\
        20 & 1.11M & 87.15 & 90.83 & 93.63 & 69.30 & 52.42 & 23.26 & 69.43 \\
        \bottomrule
    \end{tabular}%
    }
\end{table}

\subsection{Impact of Different RBF Types}
\label{app:ablation_rbf}

We investigate the superiority of the Gaussian RBF compared to other commonly used radial basis functions. Theoretically, the Gaussian function excels due to its rapid exponential decay, which produces effectively localized responses despite its noncompact support. Let $d = \lVert x - \mu_k \rVert$ denote the distance to the grid center, and $h$ denote the bandwidth parameter. 

Empirical results in Table~\ref{tab:ablation_rbf_type} are consistent with this interpretation: the Gaussian RBF achieves the highest accuracy ($70.17\%$). Heavy-tailed kernels (IMQ, IQ) produce broader overlap, whereas growing kernels (MQ, TPS) may couple responses across more distant centers. These locality differences offer a possible explanation for the measured accuracy gaps; the ablation does not isolate their causal effect on optimization.

\begin{table}[!htbp]
    \centering
    \small
    \renewcommand{\arraystretch}{1.5}
    \setlength{\tabcolsep}{4pt}
    \caption{RBF kernel ablation for GenLoRA ($r=8,g=16$) using LLaMA-3-8B
    fine-tuned on Math10K. IMQ, IQ, MQ, and TPS denote inverse multiquadric,
    inverse quadric, multiquadric, and thin plate spline, respectively.
    Avg. is the unweighted mean accuracy (\%) across the six tasks.}
    \label{tab:ablation_rbf_type}
    \resizebox{\linewidth}{!}{%
    \begin{tabular}{llrrrrrrr}
        \toprule
        RBF Type & Function $\varphi(d)$ & AddSub & MultiArith & SingleEq & SVAMP & GSM8K & AQuA & Avg. \\
        \midrule
        Inverse Multiquadric (IMQ) & $\frac{1}{\sqrt{1+(d/h)^2}}$ & 86.06 & 88.31 & 90.78 & 69.20 & 52.38 & 26.54 & 68.88 \\
        Inverse Quadric (IQ) & $\frac{1}{1+(d/h)^2}$ & 86.74 & 87.36 & 92.08 & 69.60 & 54.29 & 26.33 & 69.40 \\
        Multiquadric (MQ) & $\sqrt{1+(d/h)^2}$ & 83.07 & 87.26 & 88.73 & 67.40 & 52.85 & 20.86 & 66.70 \\
        Thin Plate Spline (TPS) & $d^2\ln(d)$ & 87.27 & 87.74 & 88.04 & 68.50 & 52.66 & 23.08 & 67.88 \\
        Gaussian (Ours) & $\exp(-(d/h)^2)$ & 88.35 & 90.67 & 93.70 & 70.80 & 53.90 & 23.62 & \textbf{70.17} \\
        \bottomrule
    \end{tabular}%
    }
\end{table}

\subsection{Necessity of the Linear Residual Term}
\label{app:ablation_residual}

Our generator architecture employs a ``Global + Local'' approximation scheme. The base linear residual term is designed to capture global, low-frequency trends in the weight manifold, thereby freeing the RBF grid to focus on fitting high-frequency, fine-grained details. As demonstrated in Table~\ref{tab:ablation_residual}, removing this base linear projection forces the RBF network to waste its representational capacity on general shifts, leading to a noticeable performance drop from $70.17\%$ to $68.00\%$.

\begin{table}[!htbp]
    \centering
    \small
    \caption{Ablation on the linear residual term in GenLoRA ($r=8,g=16$)
    using LLaMA-3-8B fine-tuned on Math10K. Avg. is the unweighted mean
    accuracy (\%) across the six tasks.}
    \label{tab:ablation_residual}
    \setlength{\tabcolsep}{4pt}
    \resizebox{\linewidth}{!}{%
    \begin{tabular}{lrrrrrrrr}
        \toprule
        Method & \#Params & AddSub & MultiArith & SingleEq & SVAMP & GSM8K & AQuA & Avg. \\
        \midrule
        GenLoRA w.o. res & 0.95M & 81.77 & 85.83 & 91.93 & 66.50 & 55.19 & 26.77 & 68.00 \\
        GenLoRA & 0.98M & 88.35 & 90.67 & 93.70 & 70.80 & 53.90 & 23.62 & \textbf{70.17} \\
        \bottomrule
    \end{tabular}%
    }
\end{table}

\subsection{Comparison of Different Generator Bases}
\label{app:ablation_bases}

We compare generator bases under comparable parameter budgets.
For the Gaussian component $F(t)=\sum_{k=1}^{K}w_k\varphi_k(t)$,
the derivative bound in Appendix~\ref{app:proof_grad_bound} yields
\begin{equation}
|F(t)-F(t')|
\leq\frac{\sqrt{2/e}}{h}\|w\|_1|t-t'|.
\end{equation}
Thus bandwidth and coefficient magnitude explicitly control sensitivity.
The mapping is nonlinear in its input but linear in its $K$ coefficients;
Gaussian responses are smooth, decay away from their centers, and admit
vectorized evaluation. Chebyshev bases have global support, with derivative
bounds that can grow quadratically with degree on $[-1,1]$, whereas
B-spline smoothness depends on degree and knot multiplicity. These
properties motivate Gaussian RBFs for compact basis-vector synthesis,
without implying universal superiority over other generators.

Table~\ref{tab:ablation_generator_bases} provides the corresponding Math10K
comparison. At $g=16$, Chebyshev, B-spline, and Gaussian RBF achieve
64.01\%, 64.54\%, and 70.17\%; at $g=32$, they achieve 66.63\%,
68.04\%, and 70.11\%, respectively. Gaussian RBF performs best in both
reported settings, complementing the structural and computational analysis.

\begin{table}[!htbp]
    \centering
    \small
    \caption{Comparison of generator bases under comparable parameter budgets
    using LLaMA-3-8B fine-tuned on Math10K. All GenLoRA configurations use
    $r=8$. Avg. is the unweighted mean accuracy (\%) across the six tasks.}
    \label{tab:ablation_generator_bases}
    \setlength{\tabcolsep}{4pt}
    \resizebox{\linewidth}{!}{%
    \begin{tabular}{clrrrrrrrr}
        \toprule
        Group & Generator Basis & \#Params & AddSub & MultiArith & SingleEq & SVAMP & GSM8K & AQuA & Avg. \\
        \midrule
        \multirow{3}{*}{$g=16$}
        & Chebyshev & 0.95M & 74.52 & 85.47 & 88.13 & 61.20 & 54.60 & 20.12 & 64.01 \\
        & B-spline & 1.06M & 76.71 & 86.33 & 88.58 & 59.80 & 55.34 & 20.47 & 64.54 \\
        & Gaussian RBF (Ours) & 0.98M & 88.35 & 90.67 & 93.70 & 70.80 & 53.90 & 23.62 & \textbf{70.17} \\
        \midrule
        \multirow{3}{*}{$g=32$}
        & Chebyshev & 1.22M & 81.15 & 85.47 & 90.13 & 64.30 & 55.31 & 23.41 & 66.63 \\
        & B-spline & 1.52M & 83.54 & 86.50 & 90.75 & 66.70 & 52.77 & 27.95 & 68.04 \\
        & Gaussian RBF (Ours) & 1.38M & 88.61 & 89.67 & 90.35 & 70.40 & 57.24 & 24.41 & \textbf{70.11} \\
        \bottomrule
    \end{tabular}%
    }
\end{table}

\subsection{Ablation Study on Latent Vectors}
\label{app:ablation_latent}

In GenLoRA, the latent vectors $Z_A$ and $Z_B$ serve as the input ``seeds'' for the RBF generators. A critical question is whether these vectors actively encode task-specific information or merely act as static anchors for the non-linear mapping. To investigate this, we conducted an ablation study on the LLaMA-3-8B model using the Math10K benchmark. We selectively froze the latent vectors during training while keeping the generator parameters tunable, using the GenLoRA configuration with $r=8$ and $g=16$ as the baseline. In these frozen-latent experiments, $Z_B$ receives Kaiming initialization before freezing, as does $Z_A$. Zero initialization applies to the B-side generator coefficients rather than to the latent vector itself, so the initial generated matrix remains $B_0=0$.

Table~\ref{tab:ablation_latent} presents the experimental results. The standard GenLoRA, where both latent vectors are learnable, achieves an average accuracy of 70.17\% with 0.98M parameters. In contrast, when we freeze both latent vectors ($Z_A, Z_B$), the performance drastically drops to 60.63\%, a decrease of 9.54 percentage points. This indicates that the generator functions alone are insufficient to capture the full adaptation dynamics; the optimization of the input latent space is crucial. Freezing only $Z_A$ results in an accuracy of 62.19\%. Since $A$ determines the projection into the low-rank subspace, its rigidity likely prevents the model from extracting optimal input features. Similarly, freezing only $Z_B$ yields an accuracy of 61.26\%. As $B$ is responsible for projecting back to the high-dimensional space, fixing its latent representation severely limits the flexibility of the update direction.

Notably, although freezing latents reduces the parameter count (e.g., to 0.39M), the trade-off in performance is severe. These findings confirm that the learnability of latent vectors is a fundamental component of GenLoRA, allowing the model to dynamically adjust the topological input to the RBF generators to synthesize optimal rank components.

\begin{table}[!htbp]
    \centering
    \small
    \caption{Ablation on latent-vector learnability using LLaMA-3-8B
    fine-tuned on Math10K. All GenLoRA configurations use $r=8,g=16$.
    ``Frozen'' indicates latent vectors not updated during training;
    \#Params counts trainable parameters. Avg. is the unweighted mean
    accuracy (\%) across the six tasks.}
    \label{tab:ablation_latent}
    \setlength{\tabcolsep}{4pt}
    \resizebox{\linewidth}{!}{%
    \begin{tabular}{lrrrrrrrr}
        \toprule
        Method & \#Params & AddSub & MultiArith & SingleEq & SVAMP & GSM8K & AQuA & Avg. \\
        \midrule
        LoRA ($r=8$) & 4.72M & 82.28 & 85.83 & 89.76 & 66.40 & 56.18 & 23.23 & 67.28 \\
        GenLoRA ($r=8,g=16$) & 0.98M & 88.35 & 90.67 & 93.70 & 70.80 & 53.90 & 23.62 & \textbf{70.17} \\
        GenLoRA (Frozen $Z_A,Z_B$) & 0.39M & 74.32 & 82.06 & 83.76 & 62.30 & 46.32 & 15.03 & 60.63 \\
        GenLoRA (Frozen $Z_A$) & 0.59M & 75.01 & 84.28 & 86.71 & 60.90 & 48.39 & 17.83 & 62.19 \\
        GenLoRA (Frozen $Z_B$) & 0.79M & 73.10 & 84.27 & 84.98 & 61.10 & 48.26 & 15.85 & 61.26 \\
        \bottomrule
    \end{tabular}%
    }
\end{table}

\section{Extended Scalability Analysis and Parameter Efficiency}
\label{app:scalability_analysis}

In our main experiments, we configured GenLoRA with a rank of $r=8$ to maintain an ``Iso-Rank'' alignment with standard baselines such as LoRA and DoRA. Additionally, we included a setting of $r=32$ to validate our core hypothesis: that GenLoRA enables the utilization of significantly higher ranks to unlock superior performance, all while maintaining a highly compact parameter footprint. Even at $r=32$, GenLoRA requires only 2.16M parameters, which is less than half of the standard LoRA at $r=8$ (4.72M), demonstrating its capability to trade minimal parameters for substantial rank increases.

We did not prioritize a strict ``Iso-Parameter'' comparison in the main results for two primary reasons. First, due to the structural differences between explicit matrix factorization and our generative approach, precise parameter alignment is geometrically constrained by integer hyperparameters such as rank and group size, making exact matching computationally infeasible without altering the model architecture. Second, and more importantly, the primary objective of this work is to demonstrate that GenLoRA can achieve superior performance by efficiently scaling up the rank with significantly \textit{fewer} parameters than standard methods, rather than merely competing at equal parameter counts.

To further investigate the scalability of our approach and validate that GenLoRA can effectively leverage a larger parameter budget, we conducted an extended experiment scaling the rank to $r=64$. As shown in Table~\ref{tab:high_rank_comparison}, we present the performance trajectory across ranks $r=8$, $32$, and $64$, with the group size consistently set to $g=16$.

Remarkably, even at the high-rank configuration of $r=64$, GenLoRA requires only 3.74M parameters, which is still approximately 20\% fewer than the standard LoRA at $r=8$ (4.72M). In terms of performance, GenLoRA exhibits a consistent upward trajectory as the rank increases. This confirms that GenLoRA successfully decouples parameter cost from rank, allowing for scalable improvements in model capacity without the prohibitive memory overhead typically associated with high-rank adaptation.

\begin{table*}[h]
\centering
\caption{Extended comparison on Mathematical Reasoning tasks (LLaMA-3-8B) across varying ranks and parameters.}
\label{tab:high_rank_comparison}
\begin{small}
\setlength{\tabcolsep}{3.5pt}
\resizebox{\linewidth}{!}{%
\begin{tabular}{lcccccccc}
\toprule
Method  & \#Params & AddSub & MultiArith & SingleEq & SVAMP & gsm8k & AQuA & Avg \\
\midrule
LoRA($r=8$)  & 4.72M & 82.28 & 85.83 & 89.76 & 66.40 & 56.18 & 23.23 & 67.28 \\
\midrule
GenLoRA($r=8$) & 0.98M & 88.35 & 90.67 & 93.70 & 70.80 & 53.90 & 23.62 & 70.17 \\
GenLoRA($r=32$)  & 2.16M & 87.09 & 92.71 & 94.49 & 70.80 & 56.18 & 25.59 & 71.14 \\
GenLoRA($r=64$) & 3.74M & 89.36 & 92.50 & 92.82 & 72.10 & 56.88 & 24.97 & \textbf{71.44} \\
\bottomrule
\end{tabular}%
}
\end{small}
\end{table*}

\section{Detailed Numerical Results for Singular Value Energy}
\label{app:energy_spectrum}

Table~\ref{tab:singular} provides the layer-wise values underlying the
singular value analysis in Figure~\ref{fig:genlora_analysis}(b), covering
the self-attention query-projection updates across all 32 layers.
The reported spectral energy is
\begin{equation}
E(\Delta W)=\sum_i\sigma_i(\Delta W)^2=\|\Delta W\|_F^2.
\end{equation}
It measures squared update magnitude, rather than the expressiveness of
the parameterization.

GenLoRA has larger update magnitudes than LoRA in the reported runs.
For example, at layer 0, GenLoRA ($r=8$, 0.98M parameters) has spectral
energy 38.03, whereas LoRA ($r=32$, 18.87M parameters) has energy 0.28.
This is approximately a $136\times$ difference in squared Frobenius norm,
or an $11.7\times$ difference in Frobenius norm. The observation shows that
fewer trainable parameters need not imply smaller learned updates in this
setting. Update magnitude also depends on scaling and optimization;
spectral energy alone cannot demonstrate greater expressiveness, more
efficient use of rank, or better generalization.

\begin{table}[htbp]
\centering
\renewcommand{\arraystretch}{1.1} 
\setlength{\tabcolsep}{1.5pt}     

\caption{Layer-wise spectral energy of the learned query-projection updates, $\sum_i\sigma_i^2=\|\Delta W\|_F^2$. Values measure squared update magnitude. The two panels report layers 0--15 and 16--31.}
\label{tab:singular}

\resizebox{\textwidth}{!}{%
\begin{tabular}{l|l|cccccccccccccccc}
\toprule
\textbf{Method} & \textbf{Config (Params)} & \textbf{L0} & \textbf{L1} & \textbf{L2} & \textbf{L3} & \textbf{L4} & \textbf{L5} & \textbf{L6} & \textbf{L7} & \textbf{L8} & \textbf{L9} & \textbf{L10} & \textbf{L11} & \textbf{L12} & \textbf{L13} & \textbf{L14} & \textbf{L15} \\
\midrule
\multirow{3}{*}{\shortstack[l]{LoRA}}
 & $r=8$ (4.72M) & 0.09 & 0.15 & 0.15 & 0.16 & 0.19 & 0.24 & 0.21 & 0.18 & 0.21 & 0.17 & 0.23 & 0.22 & 0.32 & 0.32 & 0.31 & 0.36 \\
 & $r=16$ (9.44M) & 0.14 & 0.26 & 0.25 & 0.26 & 0.31 & 0.33 & 0.33 & 0.35 & 0.30 & 0.28 & 0.36 & 0.38 & 0.52 & 0.49 & 0.45 & 0.60 \\
 & $r=32$ (18.9M) & 0.28 & 0.42 & 0.52 & 0.47 & 0.49 & 0.57 & 0.58 & 0.58 & 0.61 & 0.50 & 0.63 & 0.64 & 0.81 & 0.77 & 0.78 & 1.05 \\
\midrule
\multirow{3}{*}{\shortstack[l]{GenLoRA(Ours)}}
 & $r=8$ (0.98M) & 38.0 & 33.4 & 24.4 & 36.3 & 20.1 & 27.2 & 32.8 & 24.6 & 26.6 & 30.9 & 30.8 & 28.8 & 30.7 & 40.9 & 32.7 & 46.9 \\
 & $r=32$ (2.16M) & 198.6 & 113.5 & 93.1 & 114.7 & 90.3 & 104.9 & 105.9 & 86.3 & 103.9 & 96.6 & 104.2 & 103.0 & 119.9 & 113.0 & 100.8 & 128.1 \\
 & $r=64$ (3.74M) & 386.8 & 232.5 & 228.4 & 230.5 & 212.5 & 227.8 & 216.5 & 230.5 & 207.7 & 199.6 & 217.4 & 205.7 & 205.7 & 229.8 & 233.8 & 246.7 \\
\bottomrule
\end{tabular}%
}

\vspace{3mm} 

\resizebox{\textwidth}{!}{%
\begin{tabular}{l|l|cccccccccccccccc}
\toprule
\textbf{Method} & \textbf{Config (Params)} & \textbf{L16} & \textbf{L17} & \textbf{L18} & \textbf{L19} & \textbf{L20} & \textbf{L21} & \textbf{L22} & \textbf{L23} & \textbf{L24} & \textbf{L25} & \textbf{L26} & \textbf{L27} & \textbf{L28} & \textbf{L29} & \textbf{L30} & \textbf{L31} \\
\midrule
\multirow{3}{*}{\shortstack[l]{LoRA}}
 & $r=8$ (4.72M) & 0.40 & 0.46 & 0.46 & 0.42 & 0.51 & 0.35 & 0.43 & 0.40 & 0.44 & 0.34 & 0.39 & 0.47 & 0.51 & 0.42 & 0.48 & 0.35 \\
 & $r=16$ (9.44M) & 0.59 & 0.70 & 0.74 & 0.64 & 0.75 & 0.58 & 0.66 & 0.73 & 0.74 & 0.57 & 0.70 & 0.71 & 0.78 & 0.70 & 0.75 & 0.70 \\
 & $r=32$ (18.9M) & 0.89 & 1.10 & 1.12 & 0.99 & 1.16 & 0.98 & 1.15 & 1.07 & 1.09 & 0.97 & 1.09 & 1.03 & 1.21 & 1.23 & 1.21 & 0.97 \\
\midrule
\multirow{3}{*}{\shortstack[l]{GenLoRA(Ours)}}
 & $r=8$ (0.98M) & 37.1 & 54.7 & 53.0 & 57.5 & 57.2 & 46.4 & 49.7 & 55.6 & 36.2 & 39.5 & 54.1 & 44.1 & 66.8 & 75.8 & 55.5 & 40.7 \\
 & $r=32$ (2.16M) & 121.0 & 161.8 & 164.7 & 143.1 & 142.5 & 139.1 & 138.5 & 152.7 & 140.1 & 138.3 & 158.7 & 163.7 & 135.0 & 226.2 & 173.4 & 131.3 \\
 & $r=64$ (3.74M) & 255.7 & 334.8 & 294.6 & 293.5 & 264.9 & 312.2 & 286.1 & 319.9 & 287.0 & 287.7 & 332.6 & 297.0 & 320.4 & 426.9 & 326.8 & 256.5 \\
\bottomrule
\end{tabular}%
}
\end{table}

\section{Experiment Details}
\label{experiment details}
\subsection{Dataset Details}
\label{app:datasets}

In this section, we provide detailed descriptions of the benchmarks and datasets used to evaluate the Natural Language Generation (NLG) and Code Generation capabilities of GenLoRA.

\paragraph{Mathematical Reasoning}
We evaluate the mathematical reasoning capability of our model using the \textbf{Math10K} benchmark \citep{hu2023llm}. This benchmark consists of a curated training corpus and requires the model to perform multi-step arithmetic and logical reasoning. We assess performance across the following six sub-tasks:
\begin{enumerate}
    \item \textbf{AddSub} \citep{hosseini2014learning}: A dataset of arithmetic word problems focusing on addition and subtraction operations.
    \item \textbf{MultiArith} \citep{DBLP:journals/corr/RoyR16}: A dataset designed to test the model's ability to solve multi-step arithmetic problems involving various operations.
    \item \textbf{SingleEq} \citep{DBLP:journals/tacl/Koncel-Kedziorski15}: Comprises algebra word problems that map to single linear equations.
    \item \textbf{SVAMP} \citep{DBLP:conf/naacl/PatelBG21}: A challenge dataset created by applying variations to existing word problems to test robustness against linguistic perturbations.
    \item \textbf{GSM8K} \citep{cobbe2021training}: A dataset of high-quality, linguistically diverse grade school math word problems requiring multi-step chain-of-thought reasoning.
    \item \textbf{AQuA} \citep{ling2017program}: A large-scale dataset of algebra word problems with multiple-choice options, requiring complex reasoning and derivation.
\end{enumerate}

\paragraph{Commonsense Reasoning}
We evaluate our model on the \textbf{Commonsense170K} benchmark \citep{hu2023llm}, which aggregates multiple datasets for training and evaluation. The evaluation covers the following eight sub-tasks:
\begin{enumerate}
    \item \textbf{BoolQ} \citep{clark2019boolq}: A binary question-answering task where the goal is to determine whether the answer to a question about a given passage is ``yes'' or ``no.''
    \item \textbf{PIQA} (Physical Interaction Question Answering) \citep{bisk2020piqa}: Focuses on reasoning about physical commonsense to select the most plausible solution to a given problem.
    \item \textbf{SIQA} (Social IQa) \citep{sap2019socialiqa}: Tests social commonsense reasoning by asking questions about motivations, reactions, or outcomes in social contexts.
    \item \textbf{HellaSwag} \citep{zellers2019hellaswag}: A task designed to test contextual commonsense reasoning by selecting the most plausible continuation of a given scenario.
    \item \textbf{WinoGrande} \citep{sakaguchi2021winogrande}: A pronoun coreference resolution task that requires reasoning over ambiguous pronouns in complex sentences.
    \item \textbf{ARC-e} (AI2 Reasoning Challenge - Easy) \citep{clark2018think}: A multiple-choice question-answering task focused on elementary-level science questions.
    \item \textbf{ARC-c} (AI2 Reasoning Challenge - Challenge) \citep{clark2018think}: A more difficult subset of ARC, containing questions that require advanced reasoning and knowledge retrieval.
    \item \textbf{OBQA} (OpenBookQA) \citep{mihaylov2018can}: A question-answering task requiring reasoning and knowledge synthesis from a provided ``open book'' of science facts.
\end{enumerate}

\paragraph{Code Generation}
We assess the code generation capability of GenLoRA by fine-tuning on the \textbf{Magicoder-Evol-Instruct-110k} dataset \citep{wei2023magicoder} and evaluating on the HumanEval+ benchmark.
\begin{enumerate}
    \item \textbf{Training Data (Magicoder-Evol-Instruct-110k)}: A curated and decontaminated subset of WizardCoder \citep{luo2023wizardcoder}. It comprises approximately 110k high-quality instruction-response pairs developed via the Evol-Instruct method, designed to enhance the complexity and diversity of programming tasks.
    \item \textbf{Evaluation Benchmark (HumanEval+)}: An extended version of the HumanEval benchmark used to rigorously test functional correctness in code generation. We follow the standard evaluation protocol via the BigCode Evaluation Harness \citep{allal2022framework}, generating 50 sampled completions per problem ($n=50$) and reporting \textbf{Pass@1}, \textbf{Pass@5}, and \textbf{Pass@10} accuracy scores.
\end{enumerate}

\subsection{Baseline Details}
\label{app:baselines}

In this section, we briefly describe the baseline methods compared in our experiments:

\begin{itemize}
    \item \textbf{DoRA}~\citep{liu2024dora} decomposes pretrained weights into magnitude and direction components, optimizing the magnitude vector while utilizing LoRA for direction updates.
    \item \textbf{MELoRA}~\citep{ren2024melora} introduces mini-ensemble low-rank adapters that collectively achieve high-rank expressive power.
    \item \textbf{HiRA}~\citep{huang2025hira} utilizes the Hadamard product to enable high-rank adaptation, thereby maintaining parameter efficiency without sacrificing model capacity.
    \item \textbf{AuroRA}~\citep{dong2025aurora} addresses the low-rank bottleneck by introducing nonlinear mappings within the adapter architecture to enhance expressivity.
\end{itemize}

\subsection{Implementation Details}
\label{Implementation Details}
\paragraph{Details of Motivation Experiment.} To ensure fairness and generality, we conducted this analysis on standard, publicly available pretrained LoRA checkpoints. We selected the \texttt{hfl/llama-3-chinese-8b-instruct-lora} adapter (based on Meta Llama-3-8B) as our representative data source. Specifically, we extracted the \texttt{lora\_A} and \texttt{lora\_B} weight matrices corresponding to the Self-Attention Query Projection (\texttt{q\_proj}) from the first transformer layer to serve as the ground truth for our experiments. To reconstruct these basis vectors, we formulated a regression task where the RBF generators map a shared prototype vector derived by averaging the extracted rows to the specific target basis vectors. The training objective was to minimize the element-wise Mean Squared Error (MSE) between the synthesized vectors and the ground truth. Optimization was performed using the Adam optimizer with a learning rate of $1 \times 10^{-3}$. The training process spanned 2,000 epochs on a single NVIDIA GPU, with no weight decay applied, to ensure rigorous convergence to the target weight patterns.

\paragraph{Initialization Strategy.}
To preserve the pretrained model's behavior at the onset of fine-tuning, GenLoRA uses an asymmetric initialization with $\Delta W_0=0$:
(1) \textbf{For Matrix $A$}: The latent vector $Z_A$ is initialized using Kaiming Uniform initialization \citep{he2015delving} to introduce diverse initial variance. Correspondingly, the RBF generator weights ($\Theta_A$) are initialized from a normal distribution scaled by $1/\sqrt{d_{\text{group}}}$. Crucially, the weights for each of the $r$
generator functions are independently sampled from a normal distribution, ensuring that the synthesized basis vectors are distinct and non-degenerate.
(2) \textbf{For Matrix $B$:} The latent vector $Z_B$ is initialized
with element-wise random values rather than zeros. All B-side base-activation
coefficients $w^B_{b,i,g}$ and RBF combination coefficients $w^B_{i,g,k}$
are initialized to zero and remain trainable. Thus, even with nonzero
within-group variation in the latent inputs, the generated matrix satisfies
$B_0=0$, ensuring $\Delta W_0=B_0A_0=0$ and preserving the pretrained
model's initial output.

\paragraph{Hyperparameters and Experimental Setup}
All experiments were conducted on NVIDIA RTX A6000 GPUs. The training setup utilized the AdamW optimizer~\citep{loshchilov2017decoupled} with linear learning rate decay, a LoRA dropout rate of $0.05$, and no weight decay. For the code generation task, we employed a learning rate of $2 \times 10^{-4}$. Regarding the reasoning tasks, we set the learning rate to $1 \times 10^{-4}$ for mathematical reasoning and $5 \times 10^{-3}$ for commonsense reasoning. Across all tasks, GenLoRA and the baseline methods were applied to the query, key, and value weights, training for one epoch. The warmup ratio is 0.06 for the reasoning tasks, while code generation uses 100 warm-up steps. Our implementation builds upon the code from~\citep{hu2023llm}. Regarding the RBF generator configuration, we set the number of uniform centers to $K=15$ distributed over the fixed interval $[-3, 3]$ to align with the instance-wise normalized inputs. Accordingly, the bandwidth $h$ is explicitly defined by the grid spacing as $h = \frac{6}{K-1}$. In our comparative analysis, we primarily aligned the baseline methods to a rank of 8 to maintain a comparable parameter budget. Specifically, LoRA, DoRA, and HiRA were evaluated at rank 8. For MELoRA, we adhered to the standard setting of rank 8 with four mini-ensemble groups. Regarding AuroRA, we conducted tests at both rank 2 and rank 8 to specifically investigate its capability in mitigating the low-rank bottleneck through nonlinear mappings. For our proposed GenLoRA, we evaluated various configurations with differing ranks and group sizes. It is worth noting that GenLoRA consistently requires significantly fewer parameters than the rank 8 baselines; notably, some configurations involve fewer parameters than even the rank 2 AuroRA setting, while achieving superior performance.

\paragraph{Summary of Hyperparameters.}
Tables~\ref{tab:genlora_math_hyperparameters} and
\ref{tab:genlora_commonsense_hyperparameters} summarize the default
reasoning settings. The ranks and group counts for individual comparisons
are specified in the corresponding result tables.
Table~\ref{tab:genlora_code_hyperparameters} summarizes the code-generation
settings.

\begin{table}[!htbp]
    \centering
    \caption{Default training settings for GenLoRA on Math10K.
    The reasoning backbones are LLaMA-3-8B, Gemma-7B, and Qwen2.5-14B;
    additional ranks are specified in the corresponding result tables.}
    \label{tab:genlora_math_hyperparameters}
    \small
    \renewcommand{\arraystretch}{1.08}
    \begin{tabular*}{0.85\linewidth}{@{\extracolsep{\fill}}lc@{}}
        \toprule
        Hyperparameter & Setting \\
        \midrule
        Optimizer & AdamW \\
        Learning rate & 1e-4 \\
        Weight decay & 0.0 \\
        Warmup ratio & 0.06 \\
        Learning-rate schedule & Linear \\
        Epochs & 1 \\
        \midrule
        Default adapter rank & 8 \\
        Target modules & Q, K, V \\
        Training precision & BF16 \\
        \# GPUs per run & 1 \\
        Effective training batch size & 16 \\
        Maximum sequence length & 256 \\
        \bottomrule
    \end{tabular*}
\end{table}

\begin{table}[!htbp]
    \centering
    \caption{Default training settings for GenLoRA on Commonsense170K.
    The reasoning backbones are LLaMA-3-8B, Gemma-7B, and Qwen2.5-14B;
    additional ranks are specified in the corresponding result tables.}
    \label{tab:genlora_commonsense_hyperparameters}
    \small
    \renewcommand{\arraystretch}{1.08}
    \begin{tabular*}{0.85\linewidth}{@{\extracolsep{\fill}}lc@{}}
        \toprule
        Hyperparameter & Setting \\
        \midrule
        Optimizer & AdamW \\
        Learning rate & 5e-3 \\
        Weight decay & 0.0 \\
        Warmup ratio & 0.06 \\
        Learning-rate schedule & Linear \\
        Epochs & 1 \\
        \midrule
        Default adapter rank & 8 \\
        Target modules & Q, K, V \\
        Training precision & BF16 \\
        \# GPUs per run & 1 \\
        Effective training batch size & 16 \\
        Maximum sequence length & 256 \\
        \bottomrule
    \end{tabular*}
\end{table}

\clearpage
\begin{table}[!htbp]
    \centering
    \caption{Training and evaluation settings for GenLoRA on code generation.}
    \label{tab:genlora_code_hyperparameters}
    \small
    \renewcommand{\arraystretch}{1.08}
    \begin{tabular*}{\linewidth}{@{\extracolsep{\fill}}lc@{}}
        \toprule
        Hyperparameter & Setting \\
        \midrule
        Backbone & LLaMA-3-8B \\
        Training dataset & Magicoder-Evol-Instruct-110k \\
        Optimizer & AdamW \\
        Learning rate & 2e-4 \\
        Weight decay & 0.0 \\
        Warmup steps & 100 \\
        Learning-rate schedule & Linear \\
        Epochs & 1 \\
        Adapter dropout & 0.05 \\
        \midrule
        GenLoRA rank & 8, 32 \\
        Group count & 16 \\
        Target modules & Q, K, V \\
        RBF centers ($K$) & 15 \\
        RBF grid interval & $[-3,3]$ \\
        RBF bandwidth ($h$) & $6/(K-1)=3/7$ \\
        GPU model & NVIDIA RTX A6000 \\
        \midrule
        Evaluation benchmark & HumanEval+ \\
        Evaluation framework & BigCode Evaluation Harness \\
        Sampled completions per problem & 50 \\
        Metrics & Pass@1, Pass@5, Pass@10 \\
        \bottomrule
    \end{tabular*}
\end{table}

\end{document}